%% file: main.tex
\documentclass{article}
\usepackage[margin=1.19in]{geometry}

\usepackage{graphicx} 
\usepackage{setspace}

\usepackage{amsmath,amsfonts, amsthm}

\usepackage{algorithm}
\usepackage{algorithmic}

\usepackage{lineno}

\usepackage{amssymb}

\usepackage{color}

\usepackage[dvipsnames,svgnames,x11names,hyperref]{xcolor}

\usepackage{hyperref}
\hypersetup{colorlinks,breaklinks,      citecolor=ForestGreen,urlcolor=NavyBlue, linkcolor=NavyBlue}

\usepackage[nameinlink]{cleveref}

\usepackage{url}            
\usepackage{booktabs}       
\usepackage{amsfonts}       
\usepackage{nicefrac}       
\usepackage{microtype}      
\usepackage{mathtools}

\usepackage{xfrac}

\usepackage{enumitem}
\usepackage{pifont}

\usepackage{bm}
\usepackage{bbm}

\usepackage{wrapfig}

\allowdisplaybreaks

\input{AMS_commands.tex}
\input{tex_macros}

\title{\vspace{-2em}
\textbf{High Dimensional Robust Sparse Regression}}

\author{
 Liu Liu \\
   \texttt{liuliu@utexas.edu}
  \and
   Yanyao Shen \\
   \texttt{shenyanyao@utexas.edu}
  \and
Tianyang Li\\
   \texttt{lty@cs.utexas.edu}
   \and
Constantine Caramanis \\
\texttt{constantine@utexas.edu}
\\
\\
   The University of Texas at Austin
}

\date{}

\let\citep\cite

\begin{document}

\maketitle

\begin{abstract}%
We provide a novel -- and to the best of our knowledge, the first -- algorithm for high dimensional sparse regression with constant fraction of corruptions in explanatory and/or response variables. Our algorithm recovers the true sparse parameters
with sub-linear sample complexity,
in the presence of a constant fraction of arbitrary corruptions. Our main contribution is a robust variant of Iterative Hard Thresholding. Using this, we provide accurate estimators:
when the covariance matrix in sparse regression is identity,  our error guarantee is near information-theoretically optimal.
We then deal with robust sparse regression with unknown structured covariance matrix. We propose a filtering algorithm which
consists of a novel randomized outlier removal technique for robust sparse mean estimation that may be of interest in its own right:
the filtering algorithm is flexible enough to deal with unknown covariance.
Also, it is orderwise more efficient computationally than the ellipsoid algorithm.
Using sub-linear sample complexity, our algorithm achieves the best known (and first) error guarantee.
We demonstrate the effectiveness on large-scale sparse regression problems with arbitrary corruptions.
\end{abstract}

\input{Introduction}

\input{Background}

\input{Methodology}

\input{RobustMean}

\input{Unknown}

\section{Acknowledgments}
The authors would like to thank Simon S. Du for helpful discussions.

\clearpage

{
\small
\bibliography{Notes_Robust}
\bibliographystyle{plain}
}

\clearpage

\appendix

\setstretch{1.25}

\input{Stronger_meta.tex}

\input{Proofs}

\input{Unknown_proof}

\input{Experiment}

\end{document}

%% file: AMS_commands.tex
\newcommand{\spro}{\begin{proof}}
\newcommand{\fpro}{\end{proof}}

\usepackage{thmtools}

\makeatletter
\def\thmt@refnamewithcomma #1#2#3,#4,#5\@nil{%
  \@xa\def\csname\thmt@envname #1utorefname\endcsname{#3}%
  \ifcsname #2refname\endcsname
    \csname #2refname\expandafter\endcsname\expandafter{\thmt@envname}{#3}{#4}%
  \fi
}
\makeatother

\declaretheorem[numberwithin=section, name=Theorem,Refname={Theorem,Theorems}]{theo}
\declaretheorem[numberwithin=section, name=Lemma,Refname={Lemma,Lemmas}]{lemm}
\declaretheorem[numberwithin=section, name=Corollary,Refname={Corollary,Corollaries}]{coro}

\declaretheorem[numberwithin=section, name=Definition,Refname={Definition,Definitions}]{defi}

\declaretheorem[numberwithin=section, name=Definition,Refname={Definition,Definitions}]{definition}

\declaretheorem[numberwithin=section, name=Remark,Refname={Proposition,Propositions}]{remark}

\declaretheorem[numberwithin=section, name=Model,Refname={Model,Models}]{model} 

%% file: tex_macros.tex
\DeclareMathOperator{\Expe}{\mathbb{E}}
\DeclareMathOperator{\Real}{\mathbb{R}}

\DeclareMathOperator{\Cov}{\mathrm{Cov}}

\DeclareMathOperator*{\argmax}{arg\,max}

\DeclareMathOperator{\poly}{poly}

\newcommand{\kk}{\widetilde{k}}
\newcommand{\condition}{c_\kappa}
\newcommand{\ssmoothness}{\mu_\beta}
\newcommand{\sconvexity}{\mu_\alpha}

\newcommand{\ip}[2]{\left\langle #1, #2 \right\rangle}
\newcommand{\DSHat}{\widehat{\meandiff}}

\newcommand{\sqeps}{\sqrt{\epsilon}}
\newcommand{\eps}{\epsilon}

\newcommand{\opnorm}[1]{\left\lVert#1\right\rVert_{\rm op}}

\DeclarePairedDelimiter\abs{\lvert}{\rvert}%
\DeclarePairedDelimiter\norm{\lVert}{\rVert}%

\makeatletter
\let\oldabs\abs
\def\abs{\@ifstar{\oldabs}{\oldabs*}}
\let\oldnorm\norm
\def\norm{\@ifstar{\oldnorm}{\oldnorm*}}
\makeatother

\newcommand{\twosparseopnorm}[1]{\norm{#1}_{\rm{\kk,op}}}

\newcommand{\tr}[1]{\mathrm{Tr}\left(#1\right)}

\newcommand{\paramdiff}{\omega}

\newcommand{\meandiff}{\Delta}
\newcommand{\Event}{\mathcal{E}}
\newcommand{\Ind}{\mathcal{J}}

\newcommand{\Diff}{D}
\newcommand{\Mar}{Y}

\newcommand{\Fil}{\mathcal{F}}
\newcommand{\Good}{\mathcal{G}}

\newcommand{\Bad}{\mathcal{B}}
\newcommand{\SGood}{\mathcal{S}_\mathrm{good}}
\newcommand{\SBad}{\mathcal{S}_\mathrm{bad}}

\newcommand{\Observation}{\mathcal{S}}

\newcommand{\Input}{\mathcal{S}_\mathrm{in}}
\newcommand{\Output}{\mathcal{S}_\mathrm{out}}

\newcommand{\Hard}[1]{\mathsf{P}_{#1}}

\newcommand{\LF}{L_\mathrm{F}}
\newcommand{\Lcov}{L_\mathrm{cov}}

\newcommand{\gradsample}{{\bm {g}}}
\newcommand{\vect}{{\bm {v}}}
\newcommand{\x}{{\bm{x}}}
\newcommand{\xn}{{\widetilde{\bm{x}}}}

\newcommand{\HM}{{\bm{H}}}

\newcommand{\RSGE}{\text{RSGE}}

\newcommand{\Id}{\bm{I}_d}
\newcommand{\Sig}{\bm{\Sigma}}
\newcommand{\Sigsq}{\bm{\Sigma}^{\frac{1}{2}}}
\newcommand{\Sigsqn}{\bm{\Sigma}^{-\frac{1}{2}}}

\newcommand{\UMean}{\bm{G}}
\newcommand{\param}{\bm{\beta}}
\newcommand{\sep}{\rho_\mathrm{sep}}


%% file: Introduction.tex
\section{Introduction}
\label{sec:intro}
Learning in the presence of arbitrarily (even adversarially) corrupted outliers in the training data has a long history in Robust Statistics \citep{huber2011robust,hampel2011robust,tukey1975mathematics}, and has recently received much renewed attention. The high dimensional setting poses particular challenges as outlier removal via preprocessing is essentially impossible when the number of variables scales with the number of samples. We propose a computationally efficient estimator for outlier-robust sparse regression that has near-optimal sample complexity, and is the first algorithm resilient to a constant fraction of arbitrary outliers with corrupted covariates and/or response variables. Unless we specifically mention otherwise, all future mentions of outliers mean corruptions in covariates and/or response variables.

We assume that the authentic samples are
independent and identically distributed (i.i.d.) drawn from an uncorrupted  distribution $P$, where $P$ represents the
linear model
$y_i = \bm{x}_i^{\top}\param^{*} + \xi_i$, where
$\bm{x}_i \sim \mathcal{N} \left(0, \Sig\right)$, and $\param^* \in \Real^d$ is the true parameter
(see \Cref{sec:background} for complete details and definitions).
To model the corruptions, the adversary can choose an arbitrary $\eps$-fraction of the authentic samples, and replace
them with arbitrary values. We refer to the observations after corruption as $\eps$-corrupted samples  (\Cref{equ:contamination_model}).
This corruption model allows the adversary to select an $\eps$-fraction
of authentic samples to delete and corrupt, hence it is stronger than Huber's $\eps$-contamination model \cite{huber1964robust}, where the adversary independently corrupts each sample with probability $\eps$.

Outlier-robust regression is a classical problem within robust statistics (e.g., \cite{rousseeuw2005robust}), yet even in the low-dimensional setting, efficient algorithms robust to corruption in the covariates have proved elusive, until recent breakthroughs in \cite{ravikumar2018robust,sever2018} and \cite{klivans2018efficient}, which built on important results in Robust Mean Estimation \cite{Moitra2016FOCS,Lai2016FOCS} and Sums of Squares \cite{barak2016proofs}, respectively.

In the sparse setting, the parameter $\beta^{\ast}$ we seek to recover is also $k$-sparse, and a key goal is to provide recovery guarantees with sample complexity scaling with $k$, and \emph{sublinearly} with $d$. Without outliers, by now classical results (e.g., \cite{donoho2006compressed}) show that $n = \Omega(k\log d)$ samples from a i.i.d sub-Gaussian distribution are enough to give recovery guarantees on $\beta^{\ast}$ with and without additive noise.  These strong assumptions on the probabilistic distribution are necessary, since in the worst case, sparse recovery is known to be NP-hard \cite{Bandeira2013Certifying,zhang2014lower}.

Sparsity recovery with a constant fraction of arbitrary corruption is fundamentally hard. 
For instance, to the best of our knowledge, there's no previous work can provide exact recovery for sparse linear equations with arbitrary corruption in polynomial time. In  contrast, a simple exhaustive search can easily enumerate the samples and recover the sparse parameter in exponential time.

In this work, we seek to give an efficient, sample-complexity optimal algorithm that recovers $\beta^{\ast}$ to within accuracy depending on $\epsilon$ (the fraction of outliers). In the case of no additive noise, we are interested in algorithms that can guarantee exact recovery, independent of $\epsilon$.


\subsection{Related work}
The last 10 years have seen a resurgence in interest in robust statistics, including the problem of resilience to outliers in the data. Important problems attacked have included PCA \citep{klivans2009learning,xu2012robust,HRPCA2013,Lai2016FOCS,Moitra2016FOCS}, and more recently robust regression (as in this paper) \citep{ravikumar2018robust,sever2018,kong2018efficient,klivans2018efficient} and robust mean estimation \citep{Moitra2016FOCS,Lai2016FOCS,du2017computationally}, among others. We focus now on the recent work most related to the present paper.

\textbf{Robust regression. }
Earlier work in robust regression considers corruption only in the output, and shows that algorithms nearly as efficient as for regression without outliers succeeds in parameter recovery, even with a constant fraction of outliers \citep{li2013compressed,nguyen2013exact,bhatia2015robust,bhatia2017consistent,Price2018compressed}. Yet these algorithms (and their analysis) focus on corruption in $y$, and do not seem to extend to the setting of corrupted covariates -- the setting of this work.
In the low dimensional setting, there has been remarkable recent progress. The work in \cite{klivans2018efficient} shows that the Sum of Squares (SOS) based semidefinite hierarchy can be used for solving robust regression. 
Essentially concurrent to the SOS work, \citep{chen2017distributed,holland2017efficient,ravikumar2018robust,sever2018} use robust gradient descent for empirical risk minimization, by using robust mean estimation as a subroutine to compute robust gradients at each iteration. 
\cite{kong2018efficient} uses filtering algorithm \cite{Moitra2016FOCS} for robust regression.
Computationally, these latter works scale better than the algorithms in \cite{klivans2018efficient}, as although the Sum of Squares SDP framework gives polynomial time algorithms, they are often not practical \citep{hopkins2016fast}.

Much less appears to be known in the high-dimensional regime.
Exponential time algorithm, such as 
\cite{gao2017robust,johnson1978densest}, 
optimizes Tukey depth \cite{tukey1975mathematics, chen2015robust}. 
Their results reveal that handling a constant fraction of outliers ($\eps = const.$) is
actually minimax-optimal. 
Work in \cite{chen2013robust} first 
provided a polynomial time algorithm for this problem.
They show that replacing the standard inner product in Matching Pursuit with a trimmed version, one can recover from an $\epsilon$-fraction of outliers, with $\epsilon = O({1}/{\sqrt{k}})$.
Very recently, \cite{liu2019} considered more general sparsity constrained $M$-estimation by using a trimmed estimator in each step of gradient descent, yet the robustness guarantee $\epsilon = O({1}/{\sqrt{k}})$ is still sub-optimal.
Another approach follows as a byproduct of a recent algorithm for robust sparse mean estimation, in \cite{du2017computationally}. However, their error guarantee scales with $\norm{\param^*}_2$, and moreover, does not provide \emph{exact recovery} in the adversarial corruption case without stochastic noise (i.e., noise variance $\sigma^2 = 0$). We note that this is an inevitable consequence of their approach, as they directly use sparse mean estimation on $\{y_i \bm{x}_i\}$, rather than considering Maximum Likelihood Estimation.

\textbf{Robust  mean estimation. }
The idea in \citep{ravikumar2018robust,sever2018,kong2018efficient} is to leverage recent breakthroughs in robust mean estimation. Very recently, \citep{Moitra2016FOCS,Lai2016FOCS} provided the first robust mean estimation algorithms that can handle a constant fraction of outliers (though \cite{Lai2016FOCS} incurs
a small (logarithmic) dependence in the dimension). Following their work,
\cite{du2017computationally} extended the ellipsoid algorithm from \cite{Moitra2016FOCS}  to robust sparse mean estimation in high dimensions.
They show that $k$-sparse mean estimation in $\Real^d$ with a constant fraction of outliers can be done with $n = \Omega\left(k^2\log\left(d\right)\right)$ samples. The $k^2$ term appears to be necessary, as $n = \Omega(k^2)$ follows from an oracle-based lower bound \cite{kane2016SQ}.
\subsection{Main contributions}
\begin{itemize}[leftmargin=*]
\item Our result is  a robust variant of Iterative Hard Thresholding (IHT) \cite{IHT2009}. We provide a deterministic stability result showing that IHT works with any robust sparse mean estimation algorithm. We show our robust IHT does not accumulate the errors of a (any) robust sparse mean estimation subroutine for computing the gradient. Specifically, robust IHT produces a final solution whose error is orderwise the same as the error guaranteed by an single use of the robust mean estimation subroutine. We refer to  \Cref{def:sparse_gradient} and  \Cref{thm:Main} for the precise statement. Thus our result can be viewed as a meta-theorem that can be coupled with any robust sparse mean estimator.


\item Coupling robust IHT with a robust sparse mean estimation subroutine based on a version of the ellipsoid algorithm given and analyzed in \cite{du2017computationally}, our results \Cref{thm:coro} show that given  $\epsilon$-corrupted sparse regression samples with identity covariance, we recover $\beta^{\ast}$ within additive error $O(\sigma \epsilon)$ (which is minimax optimal \cite{gao2017robust}). The proof of the ellipsoid algorithm's performance in \cite{du2017computationally} hinges on obtaining an upper bound on the sparse operator norm (their Lemmas A.2 and A.3). As we show (see \Cref{sec:counter}), the statement of Lemma A.3 seems to be incorrect, and the general approach of upper bounding the sparse operator norm may not work. Nevertheless, the algorithm performance they claim is correct, as we show through a different avenue (see \Cref{thm:cov_imply_mean} in \Cref{subsec:correction}).

Using this ellipsoid algorithm,  
In particular, we obtain exact recovery if either the fraction of outliers goes to zero (this is just ordinary sparse regression), or in the presence of a constant fraction of outliers but with the additive noise term going to zero (this is the case of robust sparse linear equations). To the best of our knowledge, this is the first result that shows exact recovery for robust sparse linear equations with a constant fraction of outliers.
This is the content of \Cref{sec:optimal}.

\item For robust sparse regression with {\em unknown covariance matrix}, we consider the wide class of sparse covariance matrices \cite{bickel2008covariance}.
We then prove a result that may be of interest in its own right: we provide a novel robust sparse mean estimation algorithm that is based on a filtering algorithm for sequentially screening and removing potential outliers. We show that the filtering algorithm is flexible enough to deal with unknown covariance, whereas the ellipsoid algorithm cannot. It also runs a factor of $O(d^2)$ faster than the ellipsoid algorithm. If the covariance matrix is sufficiently sparse, our filtering algorithm gives a robust sparse mean estimation algorithm, that can then be coupled with our meta-theorem. Together, these two guarantee recovery of $\beta^{\ast}$ within an additive error of $O(\sigma \sqrt{\epsilon})$. 
In the case of unknown covariance, this is the best (and in fact, only) result we are aware of for robust sparse regression. We note that it can be applied to the case of known and identity covariance, though it is weaker than the optimal results we obtain using the computationally more expensive ellipsoid algorithm. Nevertheless, in both cases (unknown sparse, or known identity) the result is strong enough to guarantee exact recovery when either $\sigma$ or $\epsilon$ goes to zero.
We demonstrate the practical effectiveness of our  filtering algorithm in \Cref{sec:simulation}.
This is the content of \Cref{sec:robust-mean} and \Cref{sec:SparseCovariance}.



\end{itemize}





%% file: Background.tex
\subsection{Setup, Notation and Outline}\label{sec:background}

In this subsection, we  formally define the corruption model and the sparse regression model. We first
introduce the $\eps$-corrupted samples described above:

\begin{defi}[$\eps$-corrupted samples]
\label{equ:contamination_model}
Let $\{\bm{z}_i, i\in \Good\}$ be i.i.d. observations follow from a distribution $P$.
The $\eps$-corrupted samples $\{\bm{z}_i, i\in\Observation\}$ are generated by the following process:
an adversary chooses an arbitrary $\eps$-fraction of the samples in $\Good$ and modifies them with arbitrary values. After the corruption, we  use  $\Observation$ to denote the observations,
and use $\Bad  = \Observation \setminus \Good$ to denote the corruptions.
\end{defi}
The parameter $\epsilon$ represents the fraction of outliers. Throughout, we assume that it is a (small) constant, {\em independent of dimension or other problem parameters}.
Furthermore, we assume that the  distribution  $P$ is  the standard Gaussian-design AWGN linear model.
\begin{model}
\label{equ:stat_model}
The observations $\{\bm{z}_i = (y_i, \bm{x}_i), i \in \Good\}$ follow from the linear model $y_i = \bm{x}_i^{\top}\param^* + \xi_i$,
where $\param^*\in \Real^d$ is the model parameter, and assumed to be $k$-sparse. We assume that $\bm{x}_i \sim \mathcal{N} \left(0, \Sig\right)$
and $\xi_i \sim \mathcal{N} \left(0, \sigma^2\right)$,
where $\Sig$ is the normalized covariance matrix with  $\Sig_{jj} \leq 1$ for all $j \in [d]$.
We denote $\sconvexity$ as the smallest eigenvalue of $\Sig$, and $\ssmoothness$ as its largest eigenvalue.
They are assumed to be universal constants in this paper, and we denote the constant $\condition = \ssmoothness / \sconvexity$.
\end{model}
As in \cite{du2017computationally}, we pre-process by removing ``obvious'' outliers; we henceforth assume that all authentic and corrupted points are within a radius bounded by a polynomial in $n$, $d$ and $1/\epsilon$.
%

%
%

{\bf Notation}. We denote the hard thresholding operator of sparsity $k'$ by $\Hard{k'}$.
We define the $k$-sparse operator norm as
$\twosparseopnorm{M} = \max_{\norm{\bm{v}}_2 = 1, \norm{\bm{v}}_0 \leq \kk} |\bm{v}^{\top} M \bm{v}|$, where $M$ is not required to be positive-semidefinite (p.s.d.).
We use trace inner produce $\ip{A}{B}$ to denote $\tr{A^\top B}$.
We use $\Expe_{i \in_u \mathcal{S}}$ to denote the expectation operator obtained by the uniform distribution over all samples $i$ in a set
$\mathcal{S}$.
Finally, we use the notation $\widetilde{O}(\cdot)$ to hide the dependency on $\poly \log (1/\eps)$, and $\widetilde{\Omega}(\cdot)$ to hide the dependency on $\poly \log (k)$ in our bounds.


%% file: Methodology.tex
\section{Hard thresholding with robust gradient estimation
}\label{sec:method}

In this section, we present our method of using robust sparse gradient updates in IHT.
We then show statistical recovery guarantees given any accurate robust sparse gradient estimation, which is formally  defined in \Cref{def:sparse_gradient}.

We define the notation for the stochastic gradient $\gradsample_i$ corresponding to the $i^{th}$ point $\bm{z}_i$, and the population gradient
for ${\bm{z}_i \sim P}$ based on \Cref{equ:stat_model},
$\gradsample_i^t =  \bm{x}_i\left(\bm{x}_i^{\top} \param^{t} - y_i\right), \text{ and } \UMean^t = \Expe_{\bm{z}_i \sim P}\left(\gradsample_i^t\right),$
where $P$ is the distribution of the authentic points.
Since $\Expe_{\bm{z}_i \sim P}\left(\bm{x}_i \bm{x}_i^{\top}\right) = \Sig$,
the population mean of all authentic gradients is given by 
$\UMean^t = \Expe_{\bm{z}_i \sim P}\left(\bm{x}_i \bm{x}_i^{\top}\left(\param^{t} - \param^{*}\right)\right) = \Sig (\param^{t} - \param^{*}).$

In the \emph{uncorrupted case} where all samples $\{\bm{z}_i, i\in \Good\}$ follow from \Cref{equ:stat_model}, a single iteration of IHT updates $\beta^t$ via $
\param^{t+1} = \Hard{k'} (\param^{t} -
\Expe_{i \in_u {\Good}}\gradsample_i^t)$.
Here, the hard thresholding operator $\Hard{k'}$ selects the $k'$ largest elements in magnitude, and the parameter $k'$ is proportional to $k$ (specified in \Cref{thm:Main}).
However, given $\eps$-corrupted samples
$\{\bm{z}_i, i\in\Observation\}$ according to \Cref{equ:contamination_model}, the IHT update
based on empirical average of all gradient samples $\{\gradsample_i, i\in\Observation\}$
can be arbitrarily bad.

The key goal in this paper is to find a robust estimate $\widehat{\UMean}^t$
to replace $\UMean^t$ in each step of IHT,
 with sample complexity \emph{sub-linear} in the dimension $d$.
For instance, we consider robust sparse regression with  $\Sig = \Id$.
Then, 
$\UMean^t  = \param^{t} - \param^{*}$ is guaranteed to be $(k'+k)$-sparse 
in each iteration of IHT.
In this case, given $\eps$-corrupted samples, we can
use a robust sparse mean estimator to recover the unknown true $\UMean^t$ from
$\{\gradsample_i^t\}_{i=1}^{\abs{\Observation}}$, with sub-linear sample complexity.

More generally,
we propose Robust Sparse Gradient Estimator ({\RSGE}) for gradient estimation
given $\eps$-corrupted samples,
as defined in \Cref{def:sparse_gradient}, which guarantees that the deviation between the robust estimate  $\widehat{\UMean}\left(\param\right)$ and true $\UMean\left(\param\right)$, with sample complexity $n \ll d$.  
For a fixed $k$-sparse parameter $\param$, we drop the superscript $t$
without abuse of notation, and use $\gradsample_i$ in place of $\gradsample_i^t$, and $\UMean$ in place of $\UMean^t$;
$\UMean\left(\param\right)$ denotes the population gradient over the authentic samples' distribution $P$, at the point $\param$.

\begin{defi}[$\psi\left(\eps\right)$-{\RSGE}]
\label{def:sparse_gradient}
Given $n\left(k, d, \eps, \nu\right)$ $\eps$-corrupted samples $\{\bm{z}_i\}_{i=1}^n$
from \Cref{equ:stat_model}, we call $\widehat{\UMean}\left(\param\right)$ a $\psi\left(\eps\right)$-{\RSGE}, if
given
$\{\bm{z}_i\}_{i=1}^n$, $\widehat{\UMean}\left(\param\right)$
guarantees 
$\norm*{ \widehat{\UMean}\left(\param\right) - {\UMean}\left(\param\right)}_2^2 \leq \alpha(\eps) \norm*{{\UMean}\left(\param\right)}_2^2 + \psi\left(\eps\right)$,
with probability at least $1- \nu$.
\end{defi}

\begin{algorithm}[t]
\caption{Robust sparse regression with RSGE}
\label{alg:inexact_IHT}
\begin{algorithmic}[1]
\STATE \textbf{Input:} Data samples $\{y_i, \bm{x}_i\}_{i=1}^N$, RSGE subroutine.
\STATE \textbf{Output:} The estimation $\widehat{\param}$.
\STATE \textbf{Parameters: }Hard thresholding parameter $k'$.\\
{\kern2pt \hrule \kern2pt}
\STATE Split samples into $T$ subsets of size $n$. Initialize with $\param^0 = \bm{0}$.
\FOR {$t=0$ to $T-1$,}
\STATE At current $\param^t$, calculate all gradients for current $n$ samples: $\gradsample_i^t = \bm{x}_i\left(\bm{x}_i^{\top} \param^{t} - y_i\right)$, $i\in [n]$.
\STATE The initial input set is $\{\gradsample_i^t\}_{i=1}^n$. We use a {\RSGE}
to get $\widehat{\UMean}^t$.
\STATE Update the parameter:
$\param^{t+1} = \Hard{k'}\left(\param^{t} - \eta  \widehat{\UMean}^t\right).$
\ENDFOR

\STATE Output the estimation $\widehat{\param} = \param^{T}$.

\end{algorithmic}
\end{algorithm}

Here, we use $n\left(k, d, \eps, \nu\right)$ to denote the sample complexity as a function of
$(k, d, \eps, \nu)$, and note that the definition of RSGE does not require $\Sig$ to be identity matrix.
The parameters $\alpha(\eps)$ and $\psi\left(\eps\right)$ will be specified by concrete robust sparse mean estimators in subsequent sections.
Equipped with \Cref{def:sparse_gradient}, we propose \Cref{alg:inexact_IHT},
which takes any RSGE as a subroutine in line 7, and runs a robust variant of IHT with the estimated sparse gradient  $\widehat{\UMean}^t$ at each iteration in line 8.\footnote{Our results require sample splitting to maintain independence between subsequent iterations, though we believe this is an artifact of our analysis. Similar approach
has been used in \cite{balakrishnan2017statistical,ravikumar2018robust}
for theoretical analysis.
We do not use sample splitting technique in the experiments.}

%

\subsection{Global linear convergence and
parameter recovery guarantees}\label{sec:stablity}
In each single IHT update step,
RSGE introduces a controlled amount of error.
 \Cref{thm:Main} gives a global linear convergence guarantee for
\Cref{alg:inexact_IHT} by showing that IHT does not accumulate too much error.
In particular, we are able to recover $\param^*$ within 
  error
$O(\sqrt{\psi\left(\eps\right)})$ given any $\psi\left(\eps\right)$-{\RSGE} subroutine. 
We give the proof of \Cref{thm:Main} in \Cref{sec:proof_inexact_gradients}.
The hyper-parameter $k' = \condition^2 k $ guarantees    global linear convergence of IHT when $\condition > 1$ (when $\Sig \neq \Id$).
This setup has been used in \cite{jain2014iterative, Ping2017tight}, and is proved to be necessary in \cite{Barber2018between}.
Note that \Cref{thm:Main} is a  deterministic  stability result in nature,
and we obtain probabilistic results by certifying the RSGE condition.

\begin{theo}[Meta-theorem] \label{thm:Main}
Suppose we observe $N\left(k, d, \eps, \nu\right)$ $\eps$-corrupted samples from \Cref{equ:stat_model}.
\Cref{alg:inexact_IHT}, with
$\psi\left(\eps\right)$-{\RSGE}  defined in \Cref{def:sparse_gradient},
with step size $\eta = 1/\ssmoothness$ 
outputs $\widehat{\param}$, such that
$\norm*{\widehat{\param} - \param^*}_2 = O(\sqrt{\psi\left(\eps\right)}),$
with probability at least $1 - \nu$,
by setting $k' = \condition^2 k $ and $T = \Theta(\log({\norm{\param^*}_2}/{\sqrt{\psi(\eps)}}))$. The sample complexity is  $N\left(k, d, \eps, \nu\right) = n\left(k, d, \eps, \nu/T\right) T$.
\end{theo}

\section{Robust sparse regression with
near-optimal guarantee}
\label{sec:optimal}
In this section, we provide near optimal statistical guarantee for robust sparse regression
when the covariance matrix is identity.
Under the assumption $\Sig = \Id$,  \cite{du2017computationally} proposes a robust sparse regression estimator based on robust sparse mean estimation on $\{y_i\bm{x}_i, i \in \Observation \}$, leveraging the fact that $\Expe_{\bm{z}_i \sim P} \left(y_i \bm{x}_i\right) = \param^*$.
With sample complexity $N= \Omega\big(\frac{k^2 \log (d/\nu)}{\eps^2}\big)$, this algorithm produces a $\widetilde{\param}$ such that $
\norm*{\widetilde{\param} - \param^*}_2^2 = \widetilde{O}(\eps^2(\norm{\param^*}_2^2 + \sigma^2))$, with probability at least $1-\nu$.
Using \Cref{thm:Main}, we show that we can obtain significantly stronger statistical guarantees
which are statistically optimal; in particular, our guarantees are independent of $\norm{\param^*}_2$ and yield exact recovery when $\sigma = 0$.

\subsection{RSGE via the ellipsoid algorithm}

\begin{algorithm}[t]
\caption{Separation oracle for robust sparse estimation \cite{du2017computationally}}
\label{alg:filter-F-function}
\begin{algorithmic}[1]
\STATE \textbf{Input: }
Weights from the previous iteration  $\{w_i, i\in\Observation\}$, gradient samples  $\{\gradsample_i, i\in\Observation\}$.
\STATE \textbf{Output: } Weight $\{w_i', i\in\Observation \}$
\STATE \textbf{Parameters: }Hard thresholding parameter $\kk$, parameter $\sep$.\\
{\kern2pt \hrule \kern2pt}
\STATE Compute the weighted sample mean $\widetilde{\UMean}=\sum_{i \in \Observation}  w_i \gradsample_i $,
and $\widehat{\UMean} = \Hard{2\kk}\big(\widetilde{\UMean}\big)$.
\STATE Compute the weighted sample covariance matrix $\widehat{\bm{\Sigma}}
 = \sum_{i \in \Observation}  w_i
 \left(\gradsample_i-\widehat{\UMean}\right) \left(\gradsample_i-\widehat{\UMean}\right)^{\top}$.
\STATE Solve: $
\max_{{\HM}}  \tr{ \left(\widehat{\bm{\Sigma}} - F\left(\widehat{\UMean}\right)\right) \cdot \HM},  \quad
\text{subject to }  {\HM} \succcurlyeq 0,
 \norm{\HM}_{1,1} \le \kk,
 \tr{\HM} = 1$.
 
Let  $\lambda^*$ be the  optimal value, and ${\HM}^*$ be the corresponding solution.
\STATE \textbf{if } {$\lambda^*  \leq \sep$}
 \textbf{, then return} ``Yes''.
\STATE  \textbf{return} The hyperplane:
$\ell(w') = \ip{ \Big(
    \sum_{i \in \Observation}  w_i'
 \big(\gradsample_i-\widehat{\UMean}\big)
 \big(\gradsample_i-\widehat{\UMean}\big)^\top
 - F\big(\widehat{\UMean}\big)\Big)}{\HM^*} - \lambda^*.$
\end{algorithmic}
\end{algorithm}

More specifically, 
the ellipsoid-based robust sparse mean estimation algorithm \cite{du2017computationally} deals with outliers by trying to optimize the set of weights $\{w_i, i\in \Observation\}$ on each of the samples in $\Real^d$ -- ideally outliers would receive lower weight and hence their impact would be minimized. Since the set of weights is convex, this can be approached using a separation oracle \Cref{alg:filter-F-function}. The \Cref{alg:filter-F-function}
depends on a convex relaxation of Sparse PCA, and
the hard thresholding parameter is $\kk = k'+k$, as the population mean of all authentic gradient samples $\UMean^t$ is guaranteed to be $(k'+k)$-sparse.
In line 4 and 5, we calculate the weighted mean and covariance based on a hard thresholding operator.
In line 6 of \Cref{alg:filter-F-function},
with each call to the relaxation of Sparse PCA, we obtain an optimal value, $\lambda^{\ast}$, and optimal solution, $\HM^{\ast}$, to the problem:
\begin{align}
\label{equ:coro_trace}
\lambda^* = \max_{{\HM}} \tr{ \left(\widehat{\Sig} - F\left(\widehat{\UMean}\right)\right) \cdot \HM},  \quad
\text{subject to }  {\HM} \succcurlyeq 0,
 \norm{\HM}_{1,1} \le \kk,
 \tr{\HM} = 1.
\end{align}
Here, $\widehat{\UMean}$, $\widehat{\Sig}$  are weighted first and second order moment estimates from $\eps$-corrupted samples, and $F : \Real^d \rightarrow \Real^{d\times d}$ is a function with closed-form
\begin{align}
F(\widehat{\UMean}) = \norm*{\widehat{\UMean}}_2^2  \Id + \widehat{\UMean}\widehat{\UMean}^{\top}  + \sigma^2 \Id.
\label{equ:F_definition}
\end{align}
For \cref{equ:F_definition}, given the population mean $\UMean$, we have  $F\left(\UMean\right) =  \Expe_{\bm{z}_i \sim P}( \left(\gradsample_i - \UMean \right)\left(\gradsample_i - \UMean \right)^{\top})$, which calculates the underlying true covariance matrix.
We provide more details about the calculation of $F\left(\cdot\right)$, as well as some smoothness properties, in \Cref{sec:calculations}.

The key component in the separation oracle \Cref{alg:filter-F-function} is to use convex relaxation of Sparse PCA \cref{equ:coro_trace}.
This idea generalizes existing work on using PCA to detect outliers in low dimensional robust mean estimation  \cite{Moitra2016FOCS,Lai2016FOCS}.
To gain some intuition for \cref{equ:coro_trace}, if  $\gradsample_i$ is an outlier,
then
the optimal solution  of \cref{equ:coro_trace}, $\HM^*$, may
detect the direction of this outlier. And this outlier will be down-weighted in the output of \Cref{alg:filter-F-function} by the separating hyperplane.
Finally, \Cref{alg:filter-F-function} will 
terminate with $\lambda^*  \leq \sep$ (line 7) and output the robust sparse mean estimation of the gradients $\widehat{\UMean}$.

Indeed, the ellipsoid-algorithm-based robust sparse mean estimator gives a {\RSGE}, which we can combine with \Cref{thm:Main} to obtain stronger results. We state these as \Cref{thm:coro}.
We note again that the analysis in \cite{du2017computationally} has a flaw. Their Lemma A.3 is incorrect, as our counterexample in \Cref{sec:counter} demonstrates. We provide a correct route of analysis in  \Cref{thm:cov_imply_mean} of \Cref{sec:correct_proof}.

\subsection{Near-optimal statistical guarantees}
\begin{coro}
\label{thm:coro}
Suppose we observe $N\left(k, d, \eps, \nu\right)$ $\eps$-corrupted samples from \Cref{equ:stat_model} with $\Sig = \Id$.
By setting $\kk = k'+k$, if we use the ellipsoid
algorithm for robust sparse gradient estimation with $\sep = \Theta \big( \eps \big(\norm*{\UMean^t}_2^2 + \sigma^2\big)\big)$,
it requires 
$N\left(k, d, \eps, \nu\right)= \Omega\big(\frac{k^2 \log \left(d T/\nu\right)}{\eps^2}\big) T$ samples, and  guarantees $\psi\left(\eps\right) = \widetilde{O}\left(\eps^2\sigma^2\right)$. Hence,
\Cref{alg:inexact_IHT} outputs $\widehat{\param}$, such that
$\norm*{\widehat{\param} - \param^*}_2 = \widetilde{O}\left(\sigma \eps \right)$,
with probability at least $1 - \nu$,
by setting $T = \Theta\left(\log\left(\frac{\norm{\param^*}_2}{\eps\sigma}\right)\right)$.
\end{coro}

For a desired error level $\eps' \geq \eps$, we only require sample complexity
$N\left(k, d, \eps, \nu\right)= \Omega\big(\frac{k^2 \log \left(d T/\nu\right)}{\eps'^2}\big) T$.
Hence, we can achieve   statistical error  $\widetilde{O}\big(\sigma \big( \sqrt{{k^2 \log \left(d \right)}/{N}} \vee \eps \big)  \big)$. Our error bound is nearly optimal compared to the information-theoretically optimal ${O}\big(\sigma \big( \sqrt{{k \log \left(d \right)}/{N}} \vee \eps \big) \big)$ in \cite{gao2017robust}, as the $k^2$
term is necessary by  an oracle-based SQ lower bound \cite{kane2016SQ}.
\emph{Proof sketch of \Cref{thm:coro}}
The key to the proof relies on showing that $\lambda^{\ast}$ controls the quality of the weights of the current iteration, i.e., small $\lambda^{\ast}$ means good weights and thus a good current solution. Showing this relies on using $\lambda^{\ast}$ to control $\widehat{\Sig} - F(\widehat{\UMean})$. Lemma A.3 in \cite{du2017computationally} claims that
$ \lambda^* \geq \norm*{ \widehat{\Sig} - F(\widehat{\UMean})}_{\rm{\kk,op}}$. As we show in \Cref{sec:counter}, however, this need not hold. This is because the trace norm maximization
\cref{equ:coro_trace} is \emph{not} a valid convex relaxation for the $\kk$-sparse operator norm when the term $\widehat{\Sig} - F(\widehat{\UMean})$ is not  p.s.d. (which indeed it need not be). We provide a different line of analysis in \Cref{thm:cov_imply_mean}, essentially showing that even without the claimed (incorrect) bound, $\lambda^{\ast}$ can still provide the control we need. With the corrected analysis for $\lambda^{*}$, the ellipsoid algorithm guarantees
$\norm*{ \widehat{\UMean} - {\UMean}}_2^2
= \widetilde{O}(\eps^2 (\norm{{\param - \param^{*}}}_2^2 + \sigma^2))$ with probability at least $1-\nu$. Therefore, the algorithm provides an  $\widetilde{O}\left(\eps^2\sigma^2\right)$-{\RSGE}.

%% file: RobustMean.tex
\section{Robust sparse mean estimation via filtering}\label{sec:robust-mean}


%


\begin{algorithm}[t]
\caption{RSGE via filtering}
\label{alg:filter-mean-second}
\begin{algorithmic}[1]
\STATE \textbf{Input: }A set $\Input$.
\STATE \textbf{Output: }A set $\Output$ or sparse mean vector $\widehat{\UMean}$.
\STATE \textbf{Parameters: }Hard thresholding parameter $\kk$, parameter $\sep$.  \\
{\kern2pt \hrule \kern2pt}
\STATE Compute the sample mean $\widetilde{\UMean}=\Expe_{i \in_u \Input}\big(\gradsample_i\big)$,
and $\widehat{\UMean} = \Hard{2\kk}\big(\widetilde{\UMean}\big)$.\label{line:4}
\STATE Compute the sample covariance matrix $\widehat{\bm{\Sigma}}
 = \Expe_{i \in_u \Input}\left(\gradsample_i-\widehat{\UMean}\right) \left(\gradsample_i-\widehat{\UMean}\right)^{\top}$.

\STATE Solve the following convex program:
\begin{align}\label{equ:convex_relax}
\max_{{\HM}}  \tr{ \widehat{\Sig} \cdot \HM},  \quad
\text{subject to }  {\HM} \succcurlyeq 0,
 \norm{\HM}_{1,1} \le \kk,
 \tr{\HM} = 1.
\end{align}

Let  $\lambda^*$ be the  optimal value, and ${\HM}^*$ be the corresponding solution.

\STATE \textbf{if} {$\lambda^*  \leq \sep$}
\textbf{, then return} with $\widehat{\UMean}$.

\STATE Calculate projection score for each $i \in \Input$: 
\begin{align*}
    \tau_i =  \mathrm{Tr}\Big(\HM^* \cdot \left(\gradsample_i-\widehat{\UMean}\right)\left(\gradsample_i-\widehat{\UMean}\right)^{\top}\Big).
\end{align*}

\STATE Randomly remove a sample $r$ from $\Input$ according to \label{line:remove}
\begin{align}\label{equ:Pr_remove}
\Pr\left(\gradsample_i \text{ is removed}\right) = \frac{\tau_i}{\sum_{i \in \Input} \tau_i}.
\end{align}
\STATE  \textbf{return} the set $\Output= \Input \setminus \{r\}$.
\end{algorithmic}
\end{algorithm}

From a computational viewpoint, the time complexity of \Cref{alg:inexact_IHT} depends on the {\RSGE} in each iterate. The time complexity of the ellipsoid algorithm is indeed polynomial in the dimension, but it requires $O\big(d^2\big)$ calls to a relaxation of Sparse PCA (\cite{bubeck2015convex}). In this section, we introduce a faster algorithm as a {\RSGE}, which only requires $O\left(n\right)$ calls of Sparse PCA (recall that $n$ only scales with $k^2 \log d$). Importantly,
this {\RSGE} is flexible enough to deal with unknown covariance matrix, yet the ellipsoid algorithm cannot.
Before we move to the result for
unknown covariance matrix in \Cref{sec:SparseCovariance},
we first introduce \Cref{alg:filter-mean-second} 
and analyze its performance when the covariance is identity. These supporting Lemmas will be later used in the unknown case. 

Our proposed RSGE (\Cref{alg:filter-mean-second}) attempts to remove one outlier at each iteration, as long as a good solution has not already been identified.
It first estimates the gradient  $\widehat{\UMean}$ by
hard thresholding
(line 4)
and then estimates the corresponding
sample covariance matrix $\widehat{\bm{\Sigma}}$ (line 5).
By solving (a relaxation of) Sparse PCA, we obtain a scalar $\lambda^*$ as well as a matrix $\HM^*$.
If $\lambda^*$ is smaller than the predetermined threshold $\sep$, we have a certificate that the effect of the outliers is well-controlled (specified in \cref{equ:lambda_lower}).
Otherwise, we compute a score for each sample based on $\HM^*$, and discard one of the samples according to a
probability distribution where each sample's probability of being discarded is proportional to the score we have computed \footnote{Although we remove one sample in  \Cref{alg:filter-mean-second}, our theoretical analysis naturally extend to removing constant number of outliers.
This speeds up the algorithm in practice, yet shares the same computational complexity}.
\Cref{alg:filter-mean-second} can be used for other robust sparse functional estimation
problems
(e.g.,  robust sparse mean estimation for $\mathcal{N}\left(\bm{\mu}, \bm{I_d}\right)$, where $\bm{\mu} \in \Real^d$ is $k$-sparse).
%
%
To use \Cref{alg:filter-mean-second} as a RSGE given
$n$ gradient samples (denoted  as $\Input$), we call \Cref{alg:filter-mean-second} repeatedly on $\Input$ and then on its output, $\Output$, until it returns a robust estimator $\widehat{\UMean}$. The next theorem provides guarantees on this iterative application of \Cref{alg:filter-mean-second}.
\begin{theo}\label{thm:Filter}
Suppose we observe $n = \Omega \big(\frac{k^2 \log \left(d/\nu\right)}{\eps} \big)$ $\eps$-corrupted samples from \Cref{equ:stat_model} with $\Sig = \Id$.
Let $\Input$ be an $\eps$-corrupted set of gradient samples $\{\gradsample_i^t\}_{i=1}^n$.
By setting $\kk = k'+k$, if we run \Cref{alg:filter-mean-second} iteratively with initial set $\Input$, and subsequently on $\Output$, and use
$\sep = C_\gamma \big(\norm*{\UMean^t}_2^2 + \sigma^2\big)$, \footnote{Similar to \cite{Moitra2016FOCS,Moitra2017ICML,sever2018,du2017computationally}, our results seem to require this side information.}
then this repeated use of \Cref{alg:filter-mean-second} will stop after at most $\frac{1.1\gamma}{\gamma-1} \eps n$ iterations, and output $\widehat{\UMean}^t$, such that
$
\norm*{\widehat{\UMean}^t - \UMean^t}_2^2 = \widetilde{O}\left(\eps \left( \norm*{\UMean^t}_2^2+ \sigma^2\right)\right)$,
with probability at least $1-\nu - \exp\left(-\Theta\left(\eps n\right)\right)$. Here, $C_\gamma$ is a constant depending on $\gamma$, where $\gamma \geq 4$ is a constant.
\end{theo}

Thus, \Cref{thm:Filter} shows that with  high probability, \Cref{alg:filter-mean-second} provides a Robust Sparse Gradient Estimator where $\psi\left(\eps\right) = \widetilde{O}\left(\eps\sigma^2\right)$.
For example, we can take $\nu = d^{-\Theta(1)}$. 
Combining now with \Cref{thm:Main}, we obtain an error guarantee for robust sparse regression.

\begin{coro}
\label{thm:coro_2}
Suppose we observe $N\left(k, d, \eps, \nu\right)$ $\eps$-corrupted samples from
\Cref{equ:stat_model} with $\Sig = \Id$.
Under the same setting as \Cref{thm:Filter}, if we use \Cref{alg:filter-mean-second} for robust sparse gradient estimation,
it requires $N\left(k, d, \eps, \nu\right)= \Omega\left(\frac{k^2 \log \left(dT/\nu\right)}{\eps}\right) T$ samples, and $T = \Theta\left(\log\left(\frac{\norm{\param^*}_2}{\sigma\sqeps}\right)\right)$,
then we have
$\norm*{\widehat{\param} - \param^*}_2 = \widetilde{O}\left(\sigma\sqeps\right)$
with probability at least $1- \nu - T \exp\left(-\Theta\left(\eps n\right)\right)$.
\end{coro}

Similar to \Cref{sec:optimal},
we can achieve   statistical error   $\widetilde{O}\big(\sigma \big( \sqrt{{k^2 \log \left(d \right)}/{N}} \vee \sqeps \big)  \big)$.
The scaling of $\eps$ in \Cref{thm:coro_2} is $\widetilde{O}\left(\sqeps\right)$.
These guarantees are worse than $\widetilde{O}\left(\eps\right)$ achieved by ellipsoid methods. 
Nevertheless, this result is strong enough to guarantee exact recovery when either $\sigma$ or $\epsilon$ goes to zero.
The simulation  of robust estimation for the filtering algorithm is in \Cref{sec:simulation}.



The key step in \Cref{alg:filter-mean-second}
is  outlier removal \cref{equ:Pr_remove}
based on the solution of
Sparse PCA's convex relaxation
\cref{equ:convex_relax}. We describe the outlier removal below, and then give the proofs
 in \Cref{sec:proof_kappa_bound} and \Cref{sec:proof_Filter}.

\textbf{Outlier removal guarantees in \Cref{alg:filter-mean-second}. }
We denote  samples in the input set $\Input$ as $\gradsample_i$.
 This input set $\Input$ can be partitioned into two parts: $\SGood = \{i: i\in \Good \text{ and } i\in \Input\}$, and
$\SBad = \{i: i\in \Bad \text{ and } i\in \Input\}$. 
\Cref{thm:kappa_bound} shows that \Cref{alg:filter-mean-second} can
return a guaranteed gradient estimate, 
or  the outlier removal step \cref{equ:Pr_remove} is 
likely to discard an outlier.
The guarantee on the outlier removal step \cref{equ:Pr_remove}
hinges on the fact that 
if
$\sum_{i \in \SGood} \tau_i$  is less than ${\sum_{i \in \SBad} \tau_i}$, we can show  \cref{equ:Pr_remove} is likely to remove an outlier.

\begin{lemm}\label{thm:kappa_bound}
Suppose we observe $n = \Omega \big(\frac{k^2 \log \left(d/\nu\right)}{\eps}\big)$ $\eps$-corrupted samples from \Cref{equ:stat_model} with $\Sig = \Id$.
Let $\Input$ be an $\eps$-corrupted set  $\{\gradsample_i^t\}_{i=1}^n$. \Cref{alg:filter-mean-second} computes $\lambda^*$ that satisfies
 \begin{align}\label{equ:lambda_lower}
    \lambda^* \ge  \max_{\norm{\vect}_2=1, \norm{\vect}_0\leq \kk }\vect^{\top} \left({\Expe_{i \in_u \Input}
\left(\gradsample_i-\widehat{\UMean}\right)\left(\gradsample_i-\widehat{\UMean}\right)^\top}\right) \vect.
 \end{align}
If $\lambda^* \geq  \sep = C_\gamma \left(\norm{\UMean^t}_2^2 + \sigma^2\right)$, then with probability at least $1-\nu$, we have
$\sum_{i \in \SGood} \tau_i \leq \tfrac{1}{{\gamma}} {\sum_{i \in \Input} \tau_i}$,
where $\tau_i$ is defined in line 8, $C_\gamma$ is a constant depending on $\gamma$, and $\gamma \geq 4$ is a constant.
\end{lemm}

The proofs are  collected in \Cref{sec:proof_kappa_bound}.
In a nutshell, \cref{equ:lambda_lower}
is a natural convex relaxation for the sparsity constraint $\{\vect: \norm{\vect}_2=1, \norm{\vect}_0\leq \kk \}$. 
On the other hand, when $\lambda^* \geq  \sep$, 
the contribution of  $\sum_{i\in \SGood}\tau_i$ is  relatively small,
which can be obtained through concentration inequalities for the samples in $\SGood$.
Based on \Cref{thm:kappa_bound}, if $\lambda^* \leq \sep$,
 then the RHS of \cref{equ:lambda_lower} is bounded, leading to the error guarantee of
 $\norm*{\widehat{\UMean}^t - \UMean^t}_2^2$. On the other hand, if $\lambda^* \geq \sep$, we can show that \cref{equ:Pr_remove} is more likely to throw out samples of $\SBad$ rather than $\SGood$. Iteratively applying \Cref{alg:filter-mean-second} on the remaining samples,
we can remove those outliers with large effect, and keep the remaining outliers' effect well-controlled.
This leads to the final bounds in \Cref{thm:Filter}.

%% file: Unknown.tex
\section{Robust sparse regression with unknown covariance}\label{sec:SparseCovariance}

In this section, we consider robust sparse regression with unknown covariance matrix $\Sig$, which has additional sparsity structure. Formally, we define the sparse covariance matrices as  follows:
\begin{model}[Sparse covariance matrices]
\label{model:covariance}
In \Cref{equ:stat_model}, the authentic covariates $\{\bm{x}_i, i\in \Good\}$ are drawn from  $\mathcal{N} (0, \Sig)$.
We assume that
each row and column of $\Sig$ is $r$-sparse, but 
the positions of the non-zero entries are unknown.
\end{model}
\Cref{model:covariance} is widely studied in high dimensional statistics \cite{bickel2008covariance,el2008operator,WainwrightBook}. 
Under \Cref{model:covariance}, for the population gradient
$\UMean^t = \Expe_P\left(\x_i \x_i^{\top}\left(\param^{t} - \param^{*}\right)\right) = \Sig \paramdiff^t$, where
we use $\paramdiff^t$ to denote the $(k'+k)$-sparse vector $\param^t - \param^*$,  
 we can guarantee the
$\norm*{\UMean^t}_0 = \norm*{\Sig \paramdiff^t}_0  \leq r(k'+k)$.
Hence, we can use the filtering algorithm (\Cref{alg:filter-mean-second}) with $\kk = r(k'+k)$ as a RSGE for robust sparse regression with unknown $\Sig$. When the covariance is unknown, we cannot evaluate $F(\cdot)$ a priori, thus the ellipsoid algorithm is not applicable to this case. And we provide error guarantees as  follows.
\begin{theo}
\label{coro:unknown}
Suppose we observe $N\left(k, d, \eps, \nu\right)$ $\eps$-corrupted samples from
\Cref{equ:stat_model}, where the covariates $\bm{x}_i$'s follow from \Cref{model:covariance}.
If we use \Cref{alg:filter-mean-second} for robust sparse gradient estimation, it requires 
$\widetilde{\Omega}\left(\frac{r^2 k^2 \log \left(dT/\nu\right)}{\eps}\right) T $ samples,
and  $T = \Theta\left(\log\left(\frac{\norm{\param^*}_2}{\sigma\sqeps}\right)\right)$,
then, we have$
\norm*{\widehat{\param} - \param^*}_2 = \widetilde{O}\left(\sigma\sqeps\right)$,
with probability at least $1- \nu - T \exp\left(-\Theta\left(\eps n\right)\right)$.
\end{theo}


The proof of \Cref{coro:unknown} is collected in \Cref{sec:proof_unknown}, and the main technique
hinges on previous analysis for the identity covariance case (\Cref{thm:Filter} and \Cref{thm:kappa_bound}). 
In the case of unknown covariance, this is the best (and in fact, only) recovery guarantee we are aware of for robust sparse regression. 
We show the performance of robust estimation using our filtering algorithm with unknown covariance in \Cref{sec:simulation}, and we observe   same linear convergence as \Cref{sec:robust-mean}.

%% file: Stronger_meta.tex
\paragraph{Notation.}
In proofs, we let $\otimes$ denote the
Kronecker product, and for a vector $\bm{u}$, we denote the outer product by $\bm{u}^{\otimes 2}  = \bm{u}\bm{u}^{\top}$.
We define the infinity norm for a matrix $M$ as $\norm{M}_{\infty} = \max_{i,j} |M_{ij}|$.
Given index set $\Ind$, $\bm{v}^\Ind$ is the vector restricted to indices $\Ind$. Similarly,
$M^{\Ind\Ind}$ is the sub-matrix on indices
$\Ind \times \Ind$.
we use $\{C_j\}_{j=0}^3$
to denote constants that are independent of dimension, but whose value can change from line to line.

\section{Proofs for the meta-theorem}
\label{sec:proof_inexact_gradients}

In this section, we prove the global linear convergence guarantee
given the \Cref{def:sparse_gradient}.
In each iteration of \Cref{alg:inexact_IHT}, we use $\widehat{\UMean}^t$ to update
\begin{align*}
\param^{t+1} = \Hard{k'} \left(\param^{t} - \eta  \widehat{\UMean}^t\right),
\end{align*}
where $\eta = 1/\ssmoothness$ is a fixed step size.
Given the condition $\norm*{ \widehat{\UMean}\left(\param\right) - {\UMean}\left(\param\right)}_2^2 \leq \alpha(\eps) \norm{{\UMean}\left(\param\right)}_2^2 + \psi\left(\eps\right)$ in  RSGE's definition, we show that
\Cref{alg:inexact_IHT} linearly converges to a neighborhood around $\param^*$ with error at most $O(\sqrt{\psi\left(\eps\right)})$.

First, we introduce a supporting Lemma from \cite{Ping2017tight}, which bounds the distance
between
 $\Hard{k'} (\param^{t} - \eta  \widehat{\UMean}^t)$ and $\param^*$ in each iteration of \Cref{alg:inexact_IHT}.

\begin{lemm}[Theorem 1 in \cite{Ping2017tight}]
\label{thm:tight}
Let $\bm{z} \in \Real^d$ be an arbitrary vector and $\param^* \in \Real^d$ be any $k$-sparse signal. For any $k' \geq k$, we have the following bound:
\begin{align*}
	\norm{\Hard{k'}(\bm{z}) - \param^*}_2 \leq \sqrt{\zeta} \norm{\bm{z} - \param^*}_2,
\quad	\zeta = 1 + \frac{\rho + \sqrt{(4 + \rho ) \rho }}{2}, \quad
	\rho =  \frac{\min\{k, d-k'\}}{k' - k + \min\{k, d-k'\}}.
\end{align*}
We choose the hard thresholding parameter
$k' = k \condition^2 \ll d$, hence  $\rho = 1/\condition^2.$
\end{lemm}

\begin{theo}[\Cref{thm:Main}]
Suppose we observe $N\left(k, d, \eps, \nu\right)$ $\eps$-corrupted samples from \Cref{equ:stat_model}.
\Cref{alg:inexact_IHT}, with
$\psi\left(\eps\right)$-{\RSGE}  defined in \Cref{def:sparse_gradient},
with step size $\eta = 1/\ssmoothness$
outputs $\widehat{\param}$, such that
\begin{align*}
\norm{\widehat{\param} - \param^*}_2 = O\left(\sqrt{\psi\left(\eps\right)}\right),
\end{align*}
with probability at least $1 - \nu$,
by setting $k' = \condition^2 k $ and $T = \Theta\left(\log\left({\norm{\param^*}_2}/{\sqrt{\psi\left(\eps\right)}}\right)\right)$. The sample complexity is  $N\left(k, d, \eps, \nu\right) = n\left(k, d, \eps, \nu/T\right) T$.
\end{theo}

\spro
By splitting $N$ samples into $T$ sets (each set has sample size $n$), \Cref{alg:inexact_IHT} collects a fresh batch of samples with size $n\left(k, d, \eps, \nu/T\right)$ at each iteration $t\in [T]$.
\Cref{def:sparse_gradient} shows that for the fixed gradient expectation $\UMean^t$, the estimate for the gradient ${\UMean^t}$ satisfies:
\begin{align} \label{eqt:app-b-theorem-4-1}
\norm{ \widehat{\UMean}^t - {\UMean^t}}_2^2 \leq \alpha(\eps) \norm{\UMean^t}_2^2 + \psi(\eps)
\end{align}
with probability at least $1 - \nu/T$,
where $\alpha(\eps)$ is determined by $\eps$.

Letting $z^t = \param^{t} - \eta \widehat{\UMean}^{t}$, we study the $t$-th iteration of \Cref{alg:inexact_IHT}. Based on \Cref{thm:tight}, we have
\begin{align*}
  \norm{\param^{t+1} - \param^*}_2 &\leq \sqrt{\zeta}   \norm{\param^{t} - \eta \widehat{\UMean} - \param^*}_2 \\
   &= \sqrt{\zeta}  \norm{\param^{t} - \eta \UMean - \param^* + \eta(\UMean -  \widehat{\UMean})}_2 \\
&\leq \sqrt{\zeta}  \norm{\param^{t} - \eta \UMean - \param^*}_2 + \sqrt{\zeta} \eta \norm{\UMean -  \widehat{\UMean}}_2 \\
&\overset{(i)}\leq \sqrt{\zeta}  \norm{(\Id - \eta \bm{\Sigma}) (\param^{t} - \param^*)}_2 + \sqrt{\zeta} \eta
\sqrt{\alpha(\eps) \norm{\UMean}_2^2 + \psi(\eps)} \\
&\overset{(ii)}\leq \sqrt{\zeta}  \norm{(\Id - \eta \bm{\Sigma}) (\param^{t} - \param^*)}_2 + \sqrt{\zeta} \eta
\sqrt{\alpha(\eps)} \norm{ \bm{\Sigma} (\param^{t} - \param^*)}_2 +
 \sqrt{\zeta} \eta \sqrt{\psi(\eps)}
\end{align*}
where  (i) follows from the theoretical guarantee of RSGE, and (ii) follows from the basic inequality $\sqrt{a+b}\leq \sqrt{a} + \sqrt{b}$
for non-negative $a, b$.

By setting $\eta = 1/\ssmoothness$, we have
\begin{align}
  \norm{\param^{t+1} - \param^*}_2
&\leq \sqrt{\zeta}  \norm{(\Id - \eta \bm{\Sigma}) (\param^{t} - \param^*)}_2 + \sqrt{\zeta} \eta
\sqrt{\alpha(\eps)} \norm{ \bm{\Sigma} (\param^{t} - \param^*)}_2 +
 \sqrt{\zeta} \eta \sqrt{\psi(\eps)} \nonumber \\
 &\leq \sqrt{\zeta} ({1 - \frac{1}{\condition}})
  \norm{\param^{t} - \param^*}_2
  + \sqrt{\zeta}
\sqrt{\alpha(\eps)} \norm{\param^{t} - \param^*}_2
+ \sqrt{\zeta} \eta \sqrt{\psi(\eps)} \nonumber\\
 &\leq \sqrt{\zeta} ({1 - \frac{1}{\condition}} + \sqrt{\alpha(\eps)} )
  \norm{\param^{t} - \param^*}_2
+ \sqrt{\zeta} \eta \sqrt{\psi(\eps)} \label{equ:recursion}
\end{align}

When $\eps$ is a small enough constant,
we have
$\sqrt{\alpha(\eps)} \leq \frac{1}{2 \condition}$, then
\begin{align*}
 \sqrt{\zeta} ({1 - \frac{1}{\condition}} + \sqrt{\alpha(\eps)} ) & \leq  \sqrt{\zeta} ({1 - \frac{1}{2 \condition}} ) \\
 & \leq \sqrt{1 + \frac{\rho + \sqrt{(4 + \rho ) \rho }}{2}} ({1 - \frac{1}{2 \condition}} )
\end{align*}

Plugging in the parameter $\rho = {1}/{\condition^2}$ in \Cref{thm:tight}, we have
\begin{align*}
 \sqrt{\zeta} ({1 - \frac{1}{\condition}} + \sqrt{\alpha(\eps)} ) \leq  1 - \frac{1}{10 \condition}
\end{align*}

Together with \cref{equ:recursion}, we have the recursion
\begin{align*}
  \norm{\param^{t+1} - \param^*}_2
\leq \left(1 - \frac{1}{10 \condition}\right)
  \norm{\param^{t} - \param^*}_2
+ \sqrt{\zeta} \eta \sqrt{\psi(\eps)}.
\end{align*}

By solving this recursion and using a union bound, we have
\begin{align*}
\norm{  \param^{t} - \param^* }_2 \leq \left(  1 - \frac{1}{10 \condition} \right)^t \norm{ \param^0 - \param^*}_2 + \frac{\sqrt{\zeta} \eta \sqrt{\psi(\eps)}}{1-\left(  1 - \frac{1}{10 \condition} \right)}
\leq \left(4\alpha(\eps)\right)^t \norm{\param^*}_2^2 + 10 \condition {\sqrt{\zeta} \eta \sqrt{\psi(\eps)}},
\end{align*}
with probability at least $1- \nu$.

By the definition of $\condition$ and $\eta$, we have
$\norm{\widehat{\param} - \param^* }_2 = O\left( \frac{\sqrt{\psi\left(\eps\right)} }{\sconvexity} \right)$
\fpro



%% file: Proofs.tex
\section{Correcting Lemma A.3 in \cite{du2017computationally}'s proof}
\label{sec:counter}

A key part of the proof of the main theorem in \cite{du2017computationally} is to obtain an upper bound on the $k$-sparse operator norm. Specifically, their Lemmas A.2 and A.3 aim to show:
\begin{equation}
\label{eq:incorrect}
\lambda^{\ast} \geq \twosparseopnorm{{ \sum_{i=1}^{|\Observation|} w_i
\left(\gradsample_i-\widehat{\UMean}(w)\right)^{\otimes 2}-F\left(\widehat{\UMean}(w)\right)}}
\geq \frac{\norm{{\Hard{\kk}\left(\widetilde{\meandiff}(w)\right)}}_2^2}{5\epsilon},
\end{equation}
where $\widehat{\UMean}(w) =  \Hard{2\kk} \left(\sum_{i=1}^{|\Observation|} w_i \gradsample_i\right)$, $\widetilde{\meandiff}(w)=\sum_{i=1}^{|\Observation|} w_i \gradsample_i - \UMean$\footnote{The $\{w_i\}$ are weights, and these are defined precisely in Section \ref{sec:correct_proof}, but are not required for the present discussion or counterexample.},
and recall $\lambda^*$ is the solution to the SDP as given in
\Cref{alg:filter-mean-second}.

Lemma A.3 asserts the first inequality above, and Lemma A.2  the second. As we show below, Lemma A.3 cannot be correct. Specifically, the issue is that the quantity inside the second term in \cref{eq:incorrect} may not be positive semidefinite. In this case, the convex optimization problem whose solution is $\lambda^{\ast}$ is not a valid relaxation, and hence the $\lambda^*$ they obtain need not a valid upper bound. Indeed, we give a simple example below that illustrates precisely this potential issue.

Fortunately, not all is lost -- indeed, as our results imply, the main results in \cite{du2017computationally} is correct. The key is to show that while $\lambda^{\ast}$ does not upper bound the sparse operator norm, it does, however, upper bound the quantity
\begin{align}\label{equ:matrix_E}
\max_{\norm{\vect}_2=1, \norm{\vect}_0\leq \kk }\vect^{\top} \left({ \sum_{i=1}^{|\Observation|} w_i
\left(\gradsample_i-\widehat{\UMean}(w)\right)^{\otimes 2}-F\left(\widehat{\UMean}(w)\right)} \right) \vect.
\end{align}
We show this in \Cref{sec:correct_proof}. More specifically, in \Cref{thm:cov_imply_mean}, we replace the $\kk$-sparse operator norm in the second term of \cref{eq:incorrect} by the term in \cref{equ:matrix_E}.
We show this can be used to complete the proof in \Cref{sec:full_BDLS_coro}.

We now provide a counterexample that shows the first inequality in (\ref{eq:incorrect}) cannot hold. The main argument is that the convex relaxation for sparse PCA is a valid upper bound of the sparse operator norm only for positive semidefinite matrices. Specifically, denoting $\bm{E} = \widehat{\bm{\Sigma}} (w) - F(\widehat{\UMean}(w))$ as the matrix in \cref{equ:matrix_E}, \cite{du2017computationally} solves the following convex program:
\begin{align*}
\max_{{\HM}}  \tr{\bm{E}\cdot \HM},  \quad
\text{subject to }  {\HM} \succcurlyeq 0,
 \norm{\HM}_{1,1} \le k,
 \tr{\HM} = 1.
\end{align*}
Since
$\widehat{\bm{\Sigma}} (w) - F(\widehat{\UMean}(w))$ is no longer a p.s.d. matrix, the trace maximization above may not be a valid convex relaxation, and thus not an upper bound. Let us consider a specific example,
in robust sparse mean estimation for $\mathcal{N} \left(\mu, \Id\right)$,
where function $F\left(\cdot\right)$ is a fixed identity matrix $\Id$.
We choose $\kk=1$, $\mu = [1, 0]^{\top}$, and $d=2$.
Suppose we observe data to be $x_1 = [2.5, 0]^{\top}$, $x_2 = [0, 0]^{\top}$, and the weights for $x_1$ and $x_2$ are the same.
Then, we can compute the following matrices as:
\begin{align*}
\widehat{\bm{\Sigma}} =
           \begin{bmatrix}
    1.5625  & 0 \\
    0 & 0
  \end{bmatrix},
F  =
           \begin{bmatrix}
    1  & 0 \\
    0 & 1
  \end{bmatrix},
  \bm{E} = \widehat{\bm{\Sigma}} - F  =
           \begin{bmatrix}
    0.5625  & 0 \\
    0 & -1
  \end{bmatrix}.
  \end{align*}
It is clear that $\norm*{\widehat{\bm{\Sigma}} - F}_{\rm{\kk,op}} = 1$.
Solving the convex relaxation $\max_{{\HM}} \tr{\bm{E}\cdot \HM}$ or $\max_{{\HM}} \mathrm{Tr}({\widehat{\bm{\Sigma}} \cdot \HM})$ gives answer
  $
  \HM^*  =
           \begin{bmatrix}
    1 \quad  0;
    0 \quad 0
  \end{bmatrix}
  $
and the corresponding $\lambda^* = 0.5625$, which is clearly not an upper bound of $\norm*{\widehat{\bm{\Sigma}} - F}_{\rm{\kk,op}}$.
Hence $\lambda^* \geq \norm*{\widehat{\bm{\Sigma}} - F}_{\rm{\kk,op}}$ cannot hold in general.

\section{Covariance smoothness properties  in robust sparse mean estimation}
\label{sec:calculations}
When the covariance is identity, the ellipsoid algorithm requires
a closed form expression of the true covariance  function $F\left(\UMean\right)$. Indeed, the ellipsoid-based robust sparse mean estimation algorithm  uses  the covariance structure given by $F(\cdot)$ to detect outliers. The accuracy of robust sparse mean estimation explicitly depends on the properties of  $F\left(\UMean\right)$.
$\Lcov$ and
$\LF$ are two important properties of $F\left(\UMean\right)$,   related to its smoothness. 
 We first provide a closed-form expression for $F$, and then define precisely  smoothness parameters $\Lcov$ and
$\LF$, and show how these can be controlled.

\paragraph{Closed form expression of $F\left(\UMean\right)$.}

\begin{lemm}\label{lem:F_Id} 	
Suppose we observe i.i.d. samples $\{\bm{z}_i, i\in\Good\}$ from the distribution $P$ in \Cref{equ:stat_model} with $\Sig = \Id$, we have
the covariance of gradient as
\begin{align*}
\Cov(\gradsample) = 
\Expe_{\bm{z}_i \sim P} \left( \left(\gradsample_i - \UMean \right)\left(\gradsample_i - \UMean \right)^{\top} \right) = 
\norm{\UMean}_2^2  \Id + \UMean\UMean^{\top}  + \sigma^2 \Id.
\end{align*}
\end{lemm}

\spro
Since $\gradsample_i =  \bm{x}_i\left(\bm{x}_i^{\top} \param - y_i\right)$, and $\UMean= \Expe_{\bm{z}_i \sim P}\left(\gradsample_i\right)$ and $\Sig = \Id$, we have
\begin{align*}
\Expe_{\bm{z}_i \sim P} \left( \left(\gradsample_i - \UMean \right)\left(\gradsample_i - \UMean \right)^{\top} \right)
 &= \Expe_P\left( \left(\x\x^{\top} - \Id\right) \UMean\UMean^{\top} \left(\x\x^{\top} - \Id\right)\right) + \sigma^2 \Id \\
&= \Expe_P\left( \x\x^{\top} \UMean\UMean^{\top} \x\x^{\top}\right) - 2\Expe_P\left( \x\x^{\top} \UMean\UMean^{\top}\right) + \UMean\UMean^{\top}  + \sigma^2 \Id,
\end{align*}
where we drop $i$ in $\bm{x}_i$ without abuse of notation.

Next, we apply the Stein-type Lemma \cite{stein1981estimation} for $\x \sim \mathcal{N}\left(0, \Id \right)$, and a function $f\left(x\right)$  whose second
derivative exists:
\begin{align}
\label{equ:Stein}
\Expe\left(f\left(x\right)\x\x^{\top}\right) = \Expe\left(f\left(x\right)\right) \Id + \Expe\left(\nabla^2f\left(x\right)\right).
\end{align}

By \cref{equ:Stein}, we have
\begin{align*}
\Cov(\gradsample) = \norm{\UMean}_2^2  \Id + \UMean\UMean^{\top}  + \sigma^2 \Id.
\end{align*}
\fpro

\paragraph{Smoothness properties of $\opnorm{F}$.}

We first assume 
\begin{align}
\label{equ:LCOV}
   \Lcov = \max_{\norm{\vect}_2 = 1, \norm{\vect}_0\leq \kk}\abs{\vect^\top \Cov(\gradsample) \vect}.
\end{align}
If we define the functional $F(\cdot)$, such that $ 
F(\widehat{\UMean}) = \norm*{\widehat{\UMean}}_2^2  \Id + \widehat{\UMean} \widehat{\UMean}^{\top}  + \sigma^2 \Id$, and $F({\UMean}) = \norm{\UMean}_2^2  \Id + \UMean\UMean^{\top}  + \sigma^2 \Id$,
then we assume that there exists $\LF$ satisfying
\begin{align}
\label{equ:LF}
\opnorm{F\left(\UMean\right) - F\left(\widehat{\UMean}\right)}
\le \LF \norm{\UMean - \widehat{\UMean}}_2 +C \norm{\UMean - \widehat{\UMean}}_2^2,
\end{align}
where $C$ is a universal constant.

\begin{lemm}
\label{lem:F_smooth}
Under the same setting as  \Cref{lem:F_Id},
we have 
\begin{align*}
   \Lcov = 2\norm{\UMean}_2^2 + \sigma^2, \text{ and } \LF = 4\norm{\UMean}_2.
\end{align*}
\end{lemm}

\spro
$\Lcov$ is upper bounded by the top eigenvalue of $F\left(\UMean\right)$,
\begin{align*}
\Lcov \le \norm{F\left(\UMean\right)}_2 \le 2\norm{\UMean}_2^2 + \sigma^2.
\end{align*}

For the $\LF$ term, we have
\begin{align*}
& \opnorm{F\left(\UMean\right) - F\left(\widehat{\UMean}\right)} \\
&= \opnorm{
2\UMean^{\top}\left(\UMean - \widehat{\UMean}\right) \Id- \norm{\UMean - \widehat{\UMean}}_2^2 \Id +\UMean \left(\UMean - \widehat{\UMean}\right)^{\top} + \left(\UMean - \widehat{\UMean}\right)\UMean^{\top} -\left(\UMean - \widehat{\UMean}\right)\left(\UMean - \widehat{\UMean}\right)^{\top} } \\
&\le 4\norm{\UMean}_2 \norm{\UMean - \widehat{\UMean}}_2 +2 \norm{\UMean - \widehat{\UMean}}_2^2.
\end{align*}
Therefore, we can choose $\LF = 4\norm{\UMean}_2$ and $C=2$.

\fpro

\section{Proofs for the ellipsoid algorithm in robust sparse regression}
\label{sec:correct_proof}

In this section, we prove guarantees  for
the ellipsoid algorithm in robust sparse regression.
In the theoretical analysis of the ellipsoid algorithm, we use $\Input$ to denote the observations $\Observation$, which shares the same notations with \Cref{alg:filter-mean-second}.
We first give preliminary definitions of error terms defined on $\SGood$ and $\Input$, and then prove \Cref{thm:error_control}.
Next,  we prove concentration results for gradients of uncorrupted sparse linear regression in \Cref{thm:sparse_linear_mean_concentration}.
In \Cref{thm:cov_imply_mean}, we provide lower bounds for the $\kk$-sparse largest eigenvalue
defined in \cref{equ:matrix_E}.
Finally, we prove \Cref{thm:coro} based on previous Lemmas in \Cref{sec:full_BDLS_coro}.

\subsection{Preliminary definitions and properties related to $\SGood,\SBad$}\label{subsec:preliminary}

Here, we state again the definitions of $\SGood$, $\SBad$ and $\Input$. 
 In \Cref{alg:filter-mean-second}, we denote  the input set as $\Input$,
 which can be partitioned into two parts: $\SGood = \{i: i\in \Good \text{ and } i\in \Input\}$, and
$\SBad = \{i: i\in \Bad \text{ and } i\in \Input\}$.
Note that $\Input = \SGood \cup \SBad$, and $n = |\Input|$.
For the convenience of our analysis, we define the following error terms:
\begin{align*}
\widetilde{\meandiff}_{\SGood} &= \Expe_{i \in_u \SGood} \left(\gradsample_i\right) - \UMean, \\
\widehat{\meandiff}_{\SGood}&= \Hard{2\kk}\left(\Expe_{i \in_u \SGood} \left(\gradsample_i\right)\right) - \UMean,\\
\widetilde{\meandiff} &= \Expe_{i \in_u \Input} \left(\gradsample_i\right) - \UMean, \\
\DSHat &= \Hard{2\kk}\left(\Expe_{i \in_u \Input} \left(\gradsample_i\right)\right) - \UMean.
\end{align*}
These error terms are defined under a uniform distribution over samples, whereas previous papers using ellipsoid algorithms consider a set of balanced weighted distribution.
More specifically, the weights in our setting are defined as:
\begin{align*}
\widetilde{w}_i = \frac{1}{n}, ~~~~ \forall i \in \SGood \cup \SBad.
\end{align*}
The balanced weighted distribution is defined to satisfy:
\begin{align*}
0\le {w}_i \le \frac{1}{(1-2\epsilon)n}, ~~~~ \forall i \in \SGood \cup \SBad, ~~~~ \sum_{i\in \Input} {w}_i = 1.
\end{align*}
Notice that $\sum_{i\in \SBad} \widetilde{w}_i = O\left(\epsilon\right)$, and $\sum_{i\in\SBad}{w}_i =O\left( \frac{\epsilon}{1-2\epsilon}\right)$ with high probability, which intuitively says that both types of distributions have $O(\epsilon)$ weights over all bad samples. We are interested in considering uniform weighted samples since this formulation helps us analyze the filtering algorithm more conveniently, as we show in the following sections.

We restate the following Lemma
  which shows the connection of these different error terms.
\begin{lemm}[Lemma A.1 in \cite{du2017computationally}]
\label{thm:error_control}
Suppose $G$ is $k$-sparse. Then we have the following result:
\begin{align*}
\frac{1}{5} \norm{\DSHat}_2 \leq \norm{\Hard{k} \left(\widetilde{\meandiff}\right) }_2
\leq 4 \norm{\DSHat}_2.
\end{align*}
\end{lemm}

\subsection{Concentration bounds for gradients in  $\SGood$}

We first prove  concentration bounds for gradients for sparse linear regression in the
uncorrupted case. The following is similar to  Lemma D.1 in \cite{du2017computationally}.

\begin{lemm}\label{thm:sparse_linear_mean_concentration} 	
Suppose we observe i.i.d. gradient samples $\{\gradsample_i, i\in\Good\}$ from
\Cref{equ:stat_model} with
$\abs{\Good}=\Omega\left(\frac{k\log\left(d/\nu\right)}{\epsilon^2}\right)$.
Then,
there is a $\delta = \widetilde{O}\left(\eps\right)$, such that with probability at least $1-\nu$, for any index subset ${\Ind} \subset [d]$, $|{{\Ind}}| \le \kk$
and for any
$\Good' \subset \Good$, $\abs{\Good'} \geq (1-2\eps) \abs{\Good} $,
the following inequalities hold:
\begin{align}\label{equ:gradient_concentration}
	\norm{\Expe_{i \in_u \Good'} \left(\gradsample_i^{ {\Ind}} \right) -\UMean^{ {\Ind}} }_2 &
\leq \delta\left(\norm{{\UMean}}_2 + \sigma\right), \\
\norm{{\Expe_{i \in_u \Good'} \left(\gradsample_i^{ {\Ind}}-\UMean^{ {\Ind}}\right)} ^{\otimes 2} -  F\left(\UMean\right)^{\Ind \Ind}}_{\rm op} & \leq \delta\left(\norm{{\UMean}}_2^2 + \sigma^2\right).  \label{equ:gradientsecondmoment_concentration}
\end{align}
\end{lemm}

\spro
The main difference from their Lemma D.1 
is that we consider a uniform distribution over all samples instead of a balanced weighted distribution.
Furthermore,
eqs. (\ref{equ:gradient_concentration}) and (\ref{equ:gradientsecondmoment_concentration}) are the concentration inequalities for the mean and covariance of the collected gradient samples
$\{\gradsample_i, i\in\Good\}$ in the good set with the form:
\begin{align*}
\gradsample_i =\x_i \x_i^{\top} \UMean - \x_i \xi_i,
\end{align*}
which is equivalent to their Lemma D.1, where they consider ${y}_i \x_i  = \x_i \x_i^{\top}\param + \x_i\xi_i $. Therefore, by setting  all weights to $\frac{1}{(1-2\eps)\abs{\Good}}$ in their Lemma D.1 we obtain the desired concentration properties.
\fpro

\subsection{Relationship between the first and second moment of samples in $\Input$}
\label{subsec:correction}

In this part, we show an important connection between the covariance deviation  (the empirical covariance of $\Input$ minus the true covariance of authentic data) and the  mean deviation  (the empirical mean of $\Input$ minus the true mean of authentic data). When the  mean deviation (in $\ell_2$ sense) is large, the following Lemma implies that the covariance deviation must also be large. As a result, when the magnitude of the  covariance deviation  is  large, the current set of samples (or the current weights of all samples) needs to be adjusted; when the magnitude of the covariance deviation is small, the average of current sample set  (or the weighted sum of samples using current weights) provides a good enough estimate of the model parameter. Moreover, the same principle holds when we  use an approximation of the true covariance, which can be efficiently estimated.

Unlike Lemma A.2 in \cite{du2017computationally},
in \cref{equ:bound_a}, \cref{equ:bound_b}, we provide lower bounds for the $\kk$-sparse largest eigenvalue (rigorous definition in \cref{equ:largest_eigenvalue}), instead of the $\kk$-sparse operator norm. As we discussed in \Cref{sec:counter}, $\lambda^{*}$ is the convex relaxation of finding the $\kk$-sparse largest eigenvalue (instead of the $\kk$-sparse operator norm).
In the statement of the following Lemma, for the purpose of consistency, we consider the uniform distribution of weights. However,  the proof and results can be easily  extended to the setting with the balanced distribution of weights. This is due to the similarity between the two types of weight representation, as discussed in \Cref{subsec:preliminary}.

\begin{lemm}
\label{thm:cov_imply_mean}
Suppose $|\SBad| \leq 2\eps |\Input|$,  $\delta = \Omega\left(\eps\right)$,
and the gradient samples in $\SGood$ satisfy
\begin{align}
\norm{ \Hard{\kk} \left(\widetilde{\meandiff}_{\SGood}\right) }_2 & \leq c\left(\norm{\UMean}_2 + \sigma\right)\delta,\label{equ:condition_connection} \\
\twosparseopnorm{{\Expe_{i \in_u \SGood} \left(\gradsample_i-\UMean\right)}^{\otimes 2} -  F\left(\UMean\right)} & \leq c \left(\norm{\UMean}_2^2 + \sigma^2\right)\delta,\label{equ:condition_connection_kop}
\end{align}
where $c$ is a constant.
If $\norm{{\Hard{\kk}\left(\widetilde{\meandiff}\right)}}_2 \ge C_1\left(\norm{\UMean}_2 + \sigma\right)\delta$, where $C_1$ is a large constant, we have,
\begin{align}
 \max_{\norm{\vect}_2=1, \norm{\vect}_0\leq \kk }\vect^{\top} \left({\Expe_{i \in_u \Input}
\left(\gradsample_i-\widehat{\UMean}\right)^{\otimes 2}-F\left(\UMean\right)}\right) \vect &\geq  \frac{\norm{{\Hard{\kk}\left(\widetilde{\meandiff}\right)}}_2^2}{4\epsilon}, \label{equ:bound_a}\\
 \max_{\norm{\vect}_2=1, \norm{\vect}_0\leq \kk }\vect^{\top} \left({\Expe_{i \in_u \Input}
\left(\gradsample_i-\widehat{\UMean}\right)^{\otimes 2}-F\left(\widehat{\UMean}\right)}\right) \vect &\geq  \frac{\norm{{\Hard{\kk}\left(\widetilde{\meandiff}\right)}}_2^2}{5\epsilon}.\label{equ:bound_b}
\end{align}
\end{lemm}

\begin{proof}
We focus on the  $\kk$-sparse largest eigenvalue  (rigorous definition in \cref{equ:largest_eigenvalue}),
which is the correct route of analysis the convex relaxation of Sparse PCA.

Let $\Ind = \argmax_{\Ind' \subset [d], |{{\Ind'}}| \le \kk} \norm{{\widetilde{\meandiff}^{\Ind'}}}_2$.
Then $\widetilde{\meandiff}^{\Ind} = \norm{{\Hard{\kk}\left(\widetilde{\meandiff}\right)}}_2 \ge C_1\left(\norm{\UMean}_2 + \sigma\right)\delta$ according to the assumption. Using $\abs{\Input}$ to denote the size of $\Input$,
we have a lower bound for the sum over bad samples:
\begin{align*}
 \norm{ \frac{1}{\abs{\Input}}  \sum_{i\in\SBad}  \left(\gradsample_i^\Ind - \UMean^\Ind\right)}_2
 &= \norm{\widetilde{\meandiff}^{\Ind}- \frac{1}{\abs{\Input}} \sum_{i\in\SGood}\left(\gradsample_i^\Ind - \UMean^\Ind\right)}_2 \\
& \geq \norm{\widetilde{\meandiff}^{\Ind}}_2 -\norm{  \frac{1}{\abs{\Input}} \sum_{i\in\SGood}\left(\gradsample_i^\Ind - \UMean^\Ind\right)}_2 \\
& \overset{(i)}\geq \norm{\widetilde{\meandiff}^{\Ind}}_2 - c\left(\norm{\UMean}_2 + \sigma\right)\delta \\
& \overset{(ii)}\geq \frac{\norm{\widetilde{\meandiff}^{\Ind}}_2}{1.1},
\end{align*}
where (i) follows from \cref{equ:condition_connection} and  the assumptions;
(ii) follows from that we choose $C_1$ large enough.

By p.s.d.-ness of covariance matrices, we have
\begin{align*}
	\frac{1}{|\SBad|}\sum_{i\in\SBad}\left(\gradsample_i^\Ind-\UMean^{\Ind}\right)  \left(\gradsample_i^\Ind-\UMean^{\Ind}\right)^{\top}  \succcurlyeq \left(\frac{1}{|\SBad|} \sum_{i\in\SBad} \left(\gradsample_i^\Ind-\UMean^{\Ind}\right)\right)^{\otimes 2}.
\end{align*}
Therefore, because $\abs{\SBad} \le 2\epsilon |\Input|$, we have
\begin{align}\label{equ:bad_op}
\norm{{\frac{1}{\abs{\Input}} \sum_{i\in \SBad} \left(\gradsample_i^\Ind-\UMean^{\Ind}\right)^{\otimes 2}}}_{\footnotesize{\mbox{op}}}
\geq
\frac{ \norm{ \frac{1}{\abs{\Input}} \sum_{i\in\SBad}   \left(\gradsample_i^\Ind - \UMean^\Ind\right)}_2^2 }{2\eps}
\geq
\frac{ \norm{\widetilde{\meandiff}^{\Ind}}_2^2} {2.5\epsilon}.
\end{align}
With a lower bound of this submatrix of the covariance matrix, we define a vector $\vect_0 \in \Real^{\kk}$ as follows:
\begin{align}
\vect_0 = \argmax_{\norm{\vect}_2 = 1}  \vect^{\top}\left({\sum_{i\in\SBad}\frac{1}{\abs{\Input}}  \left(\gradsample_i^\Ind-\UMean^\Ind\right)^{\otimes 2}}\right) \vect.
\label{equ:largest_eigenvalue}
\end{align}

For this $\vect_0$, we have
\begin{align}
&\vect_0^{\top}\left({\frac{1}{\abs{\Input}}  \sum_{i=1}^{\abs{\Input}}\left(\gradsample_i^\Ind-\UMean^\Ind\right)^{\otimes 2} - F\left(\UMean\right)^{\Ind\Ind} } \right)\vect_0 \nonumber\\
&\ge  \vect_0^{\top}\left({\frac{1}{\abs{\Input}}  \sum_{i\in\SBad} \left(\gradsample_i^\Ind-\UMean^\Ind\right)^{\otimes 2}}\right) \vect_0 \nonumber \\
&- \norm{\frac{1}{\abs{\Input}} \sum_{i\in\SGood}\left(\gradsample_i^\Ind-\UMean^\Ind\right)^{\otimes 2} - \frac{|\SGood|}{\abs{\Input}}F\left(\UMean\right)^{\Ind\Ind}}_{\footnotesize{\mbox{op}}} -\norm{\frac{|\SBad|}{\abs{\Input}}F\left(\UMean\right)^{\Ind\Ind}}_{\footnotesize{\mbox{op}}} \nonumber\\
&\overset{(i)}\geq \frac{\norm{\widetilde{\meandiff}^{\Ind}}^2}{2.5\epsilon} -c\left(\norm{\UMean}_2^2 + \sigma^2\right)\delta-
2\epsilon ( \norm{\UMean}_2^2 + \sigma^2)
\nonumber\\
&\overset{(ii)}\geq \frac{\norm{\widetilde{\meandiff}^{\Ind}}^2}{3\epsilon}, \label{equ:bound_3}
\end{align}
where (i) follows from \cref{equ:condition_connection_kop}  and \cref{equ:bad_op}; (ii)
follows from the assumption that $\epsilon$ is sufficiently small.



Applying \cref{equ:bound_3} on our target ${\Expe_{i\in_u \Input}
\left(\gradsample_i-\widehat{\UMean}\right)^{\otimes 2}-F\left(\UMean\right)}$, we have
\begin{align}
&\vect_0^{\top}\left({\frac{1}{\abs{\Input}}\sum_{i=1}^{\abs{\Input}} \left(\gradsample_i^\Ind-\widehat{\UMean}^\Ind\right)^{\otimes 2} - F\left(\UMean\right)^{\Ind\Ind} } \right)\vect_0 \nonumber\\
&= \vect_0^{\top}\left(\frac{1}{\abs{\Input}}\sum_{i=1}^{\abs{\Input}} \left(\gradsample_i^\Ind-\UMean^\Ind\right)^{\otimes 2} - F\left(\UMean\right)^{\Ind\Ind} -\widehat{\meandiff}^{\Ind}\left(\widetilde{\meandiff}^{\Ind}\right)^\top -\widetilde{\meandiff}^{\Ind}\left(\widehat{\meandiff}^{\Ind}\right)^\top
+ \left(\widehat{\meandiff}^{\Ind}\right)^{\otimes 2} \right)\vect_0 \nonumber\\
& \overset{(i)}\ge  \vect_0^{\top}\left({\frac{1}{\abs{\Input}} \sum_{i=1}^{\abs{\Input}} \left(\gradsample_i^\Ind-\UMean^\Ind\right)^{\otimes 2} - F\left(\UMean\right)^{\Ind\Ind} } \right)\vect_0 - 24\left(\norm{\widetilde{\meandiff}^{\Ind}}_2^2\right) \nonumber\\
&\overset{(ii)}\ge \frac{\norm{\widetilde{\meandiff}^{\Ind}}_2^2}{4\epsilon}, \label{equ:bound_4}
\end{align}
where (i) follows from \Cref{thm:error_control}; (ii) follows from \cref{equ:bound_3}
and $\eps$ is sufficiently small.
By a construction $\vect = (\vect_0, \bm{0}_{d-\kk})^\top$, it is easy to see that  $\vect_0$ provides a lower bound for the maximum of $\{\vect: \norm{\vect}_2=1, \norm{\vect}_0\leq \kk\}$ in \cref{equ:bound_a}.

By \cref{equ:bound_4}, we already know that
\begin{align*}
\vect_0^{\top}\left( \frac{1}{\abs{\Input}} {\sum_{i=1}^{\abs{\Input}}\left(\gradsample_i^\Ind-\widehat{\UMean}^\Ind\right)^{\otimes 2} - F\left(\UMean\right)^{\Ind\Ind} } \right)\vect_0
\ge  \frac{\norm{\widetilde{\meandiff}^{\Ind}}_2^2}{4\epsilon}.
\end{align*}

By our assumptions on $F$,  we have \begin{align*}
\twosparseopnorm{F\left(\UMean\right)-F\left(\widehat{\UMean}\right)} &\le L_F\norm{\widehat{\meandiff}}_2 + C\norm{\widehat{\meandiff}}_2^2 \\
& \overset{(i)}\le 5L_F\norm{\widetilde{\meandiff}^\Ind}_2 + 5C\norm{\widetilde{\meandiff}^\Ind}_2^2,
\end{align*}
where (i) follows from \Cref{thm:error_control}.
Since $\delta = \Omega\left(\eps\right)$,  we obtain \cref{equ:bound_b}  by using the triangle inequality.

\end{proof}

\subsection{Proof of \Cref{thm:coro}}
\label{sec:full_BDLS_coro}

Equipped with \Cref{thm:error_control}, \Cref{thm:sparse_linear_mean_concentration}
and \Cref{thm:cov_imply_mean}, we can now prove
\Cref{thm:coro}.

\begin{coro}[\Cref{thm:coro}]
Suppose we observe $N\left(k, d, \eps, \nu\right)$ $\eps$-corrupted samples from \Cref{equ:stat_model} with $\Sig = \Id$.
By setting $\kk = k'+k$, if we use the ellipsoid
algorithm for robust sparse gradient estimation with $\sep = \Theta \big( \eps \big(\norm*{\UMean^t}_2^2 + \sigma^2\big)\big)$,
it requires 
$N\left(k, d, \eps, \nu\right)= \Omega\big(\frac{k^2 \log \left(d T/\nu\right)}{\eps^2}\big) T$ samples, and  guarantees $\psi\left(\eps\right) = \widetilde{O}\left(\eps^2\sigma^2\right)$. Hence,
\Cref{alg:inexact_IHT} outputs $\widehat{\param}$, such that
\begin{align*}
\norm{\widehat{\param} - \param^*}_2 = \widetilde{O}\left(\sigma \eps \right),
\end{align*}
with probability at least $1 - \nu$,
by setting $T = \Theta\left(\log\left(\frac{\norm{\param^*}_2}{\eps\sigma}\right)\right)$.
\end{coro}
\spro
We consider only the $t$-th iteration,
and thus  omit $t$ in $\gradsample_i^t$ and $\UMean^t$.
The function $F\left(\UMean\right)$  is given by
$
F\left(\UMean\right) = \norm{\UMean}_2^2  \Id + \UMean\UMean^{\top}  + \sigma^2 \Id,
$
as in  \Cref{sec:calculations}.
The accuracy in robust sparse estimation on gradients depends on two parameters for $F\left(\UMean\right)$: $\Lcov = 2 \norm{\UMean}_2^2 + \sigma^2$, and $\LF = 4\norm{\UMean}_2$, which are calculated in \Cref{sec:calculations}.

Under the statistical model and the contamination model described in \Cref{thm:Main}, we can set the
parameters $\sep = \Theta(\eps \big(\norm{\UMean^t}_2^2 + \sigma^2\big))$ in \Cref{alg:filter-F-function} by the calculation of $\Lcov$ and $\LF$

The ellipsoid algorithm considers all possible sample weights in a convex set and finds the optimal weight for each sample. The algorithm iteratively uses a separation oracle \Cref{alg:filter-F-function}, which solves the convex relaxation of Sparse PCA at each iteration:
\begin{align}
\label{equ:trace_proof}
\lambda^* = \max_{{\HM}} \tr{ \left(\widehat{\bm{\Sigma}} - F\left(\widehat{\UMean}\right)\right) \cdot \HM},  \quad
\text{subject to }  {\HM} \succcurlyeq 0,
 \norm{\HM}_{1,1} \le \kk,
 \tr{\HM} = 1.
\end{align}

To prove the Main
Theorem (Theorem 3.1) in \cite{du2017computationally},
the only modification is to replace the lower bound of $\lambda^*$ in their
Lemma A.3.

A weighted version of \Cref{thm:cov_imply_mean} implies that if the mean deviation is large, then
\begin{align}\label{equ:true_lower_1}
\max_{\norm{\vect}_2=1, \norm{\vect}_0\leq \kk }\vect^{\top} \left({ \sum_{i=1}^{|\Input|} w_i
\left(\gradsample_i-\widehat{\UMean}(w)\right)^{\otimes 2}-F\left(\widehat{\UMean}(w)\right)}\right) \vect \geq  \frac{\norm{{\Hard{\kk}\left(\widetilde{\meandiff}(w)\right)}}_2^2}{5\epsilon},
\end{align}
where $\widehat{\UMean}(w) =  \Hard{2\kk} \left(\sum_{i=1}^{|\Input|} w_i \gradsample_i\right)$, and $\widetilde{\meandiff}(w)=\sum_{i=1}^{|\Input|} w_i \gradsample_i - \UMean$.
Then, $\lambda^*$ in the ellipsoid algorithm satisfies
\begin{align}\label{equ:true_lower_2}
\lambda^* \geq \max_{\norm{\vect}_2=1, \norm{\vect}_0\leq \kk }\vect^{\top} \left({ \sum_{i=1}^{|\Input|} w_i
\left(\gradsample_i-\widehat{\UMean}(w)\right)^{\otimes 2}-F\left(\widehat{\UMean}(w)\right)}\right)\vect,
\end{align}
since $\lambda^*$ is the solution to the trace norm maximization \cref{equ:trace_proof}, which is the convex relaxation of finding the $\kk$-sparse largest eigenvalue.

Combining \cref{equ:true_lower_1} and \cref{equ:true_lower_2}, we have
\begin{align}
\lambda^* \geq  \frac{\norm{{\Hard{\kk}\left(\widetilde{\meandiff}(w)\right)}}_2^2}{5\epsilon},
\end{align}
which recovers the correctness of the separation oracle
in the ellipsoid algorithm, and their Main
Theorem (Theorem 3.1).

Finally, the ellipsoid algorithm guarantees that,
with  sample complexity
$\Omega\left(\frac{k^2 \log (d/\nu)}{\eps^2}\right)$, the estimate $\widehat{\UMean}$ satisfies
\begin{align}
\label{equ:BDLS_gradient}
\norm{ \widehat{\UMean} - {\UMean}}_2^2=
\widetilde{O}\left(\eps^2 \left(\LF^2 + \Lcov\right)\right)
= \widetilde{O}\left(\eps^2 \left(\norm{{\UMean}}_2^2 + \sigma^2\right)\right),
\end{align}
with probability at least $1-\nu$.
This exactly gives us a $\widetilde{O}\left(\eps^2\sigma^2\right)$-{\RSGE}. Hence, we can apply \cref{equ:BDLS_gradient} as the {\RSGE} in
\Cref{thm:Main} to prove \Cref{thm:coro}.
\fpro

\section{Outlier removal guarantees in the  filtering algorithm}
\label{sec:proof_kappa_bound}

In this section, we consider a single iteration of \Cref{alg:inexact_IHT}, and prove \Cref{thm:kappa_bound} at the $t$-th step. For clarity, we omit the superscript $t$ in both $\gradsample_i^t$ and $\UMean^t$.

In order to show guarantees for \Cref{thm:kappa_bound},
we leverage previous results
\Cref{thm:sparse_linear_mean_concentration}
and \Cref{thm:cov_imply_mean}.
We state \Cref{thm:concentration_sqeps} as
a modification of \Cref{thm:sparse_linear_mean_concentration}
by replacing $\eps$ by $\sqeps$, using
 concentration results in \Cref{thm:sparse_linear_mean_concentration}, and replacing $\eps$ by $\sqeps$.
We state \Cref{lem:cov_sqeps}
as a modification of \Cref{thm:cov_imply_mean}
by replacing $\delta = \Omega\left(\eps\right)$
with $\delta = \Omega\left(\sqeps\right)$,
since  the results for $\delta = \Omega\left(\eps\right)$
implies the results for  $\delta = \Omega\left(\sqeps\right)$.


The reason we modify the above is to prove guarantees for our computationally more efficient RSGE described in \Cref{alg:filter-mean-second}. Our motivation for calculating the score for each sample according to
$\tau_i =  \mathrm{Tr}(\HM^* \cdot (\gradsample_i-\widehat{\UMean})(\gradsample_i-\widehat{\UMean})^{\top})$ is to make sure that all the scores $\tau_i$ are positive (notice that the scores calculated based on the original non-p.s.d matrix may be negative). Based on this, we show that the sum of scores over all bad samples is  a large constant ($>1$) times larger than the sum of scores over all good samples. When finding an upper bound for $\sum_{i\in\SGood} \tau_i$, we compromise an $\epsilon$ factor in the value of $\lambda^{*}$, which results in an $\sqrt{\epsilon}$ factor in the recovery guarantee.

As described above, we immediately have
\Cref{thm:concentration_sqeps} and
\Cref{lem:cov_sqeps}
given the proofs in
\Cref{sec:correct_proof}. Note that we still use the same definitions $\widetilde{\meandiff}_{\SGood}$ and $\widetilde{\meandiff}$ on set $\SGood$ and $\Input$ respectively as in \Cref{subsec:preliminary}.

\begin{lemm}
\label{thm:concentration_sqeps} 	
Suppose we observe i.i.d. gradient samples $\{\gradsample_i, i\in\Good\}$ from
\Cref{equ:stat_model} with
$\abs{\Good} =\Omega\left(\frac{k\log\left(d/\nu\right)}{\epsilon}\right)$.
Then there is a $\delta = \widetilde{O}\left(\sqeps\right)$ that with probability at least $1-\nu$, we have for any subset ${\Ind} \subset [d]$, $|{{\Ind}}| \le \kk$,
and for any
$\Good' \subset \Good$, $\abs{\Good'} \geq (1-2\eps) \abs{\Good} $,
the following inequalities hold:
\begin{align}
	\norm{\Expe_{i \in_u \Good'} \left(\gradsample_i^{ {\Ind}} \right) -\UMean^{ {\Ind}} }_2 &
\leq \delta\left(\norm{{\UMean}}_2 + \sigma\right), \label{equ:gradient_concentration_sqeps} \\
\norm{{\Expe_{i \in_u \Good'} \left(\gradsample_i^{ {\Ind}}-\UMean^{ {\Ind}}\right)} ^{\otimes 2} -  F\left(\UMean\right)^{\Ind \Ind}}_{\rm op} & \leq \delta\left(\norm{{\UMean}}_2^2 + \sigma^2\right) . \label{equ:gradientsecondmoment_concentration_sqeps}
\end{align}
\end{lemm}

\begin{lemm}
\label{lem:cov_sqeps}
Suppose $|\SBad| \leq 2\eps |\Input|$, $\delta = \Omega\left(\sqeps\right)$,
and the gradient samples in $\SGood$ satisfy
\begin{align}
\norm{ \Hard{\kk} \left(\widetilde{\meandiff}_{\SGood}\right) }_2 & \leq c\left(\norm{\UMean}_2 + \sigma\right)\delta,\label{equ:condition_connection_sqeps} \\
\twosparseopnorm{{\Expe_{i \in_u \SGood} \left(\gradsample_i-\UMean\right)}^{\otimes 2} -  F\left(\UMean\right)} & \leq c \left(\norm{\UMean}_2^2 + \sigma^2\right)\delta, \label{equ:condition_connection_kop_sqeps}
\end{align}
where $c$ is a constant.
If $\norm{{\Hard{\kk}\left(\widetilde{\meandiff}\right)}}_2 \ge C_1\left(\norm{\UMean}_2 + \sigma\right)\delta$, where $C_1$ is a constant.
Then we have,
\begin{align}
 \max_{\norm{\vect}_2=1, \norm{\vect}_0\leq \kk }\vect^{\top} \left({\Expe_{i\in_u \Input}
\left(\gradsample_i-\widehat{\UMean}\right)^{\otimes 2}-F\left(\UMean\right)}\right) \vect &\geq  \frac{\norm{{\Hard{\kk}\left(\widetilde{\meandiff}\right)}}_2^2}{4\epsilon}, \label{equ:bound_a_sqeps}\\
 \max_{\norm{\vect}_2=1, \norm{\vect}_0\leq \kk }\vect^{\top} \left({\Expe_{i \in_u \Input}
\left(\gradsample_i-\widehat{\UMean}\right)^{\otimes 2}-F\left(\widehat{\UMean}\right)}\right) \vect &\geq  \frac{\norm{{\Hard{\kk}\left(\widetilde{\meandiff}\right)}}_2^2}{5\epsilon}.\label{equ:bound_b_sqeps}
\end{align}
\end{lemm}
%
%
%
%

By \Cref{thm:concentration_sqeps},
 \cref{equ:condition_connection_sqeps} and \cref{equ:condition_connection_kop_sqeps} in
 \Cref{lem:cov_sqeps}
 are satisfied,
 provided that we have $\abs{\Good} =\Omega\left(\frac{k\log\left(d/\nu\right)}{\epsilon}\right)$.
Now, equipped with \Cref{thm:concentration_sqeps} and \Cref{lem:cov_sqeps},
the effect of good samples can be controlled by concentration inequalities.
Based on these,
we are ready to prove \Cref{thm:kappa_bound}.

\begin{lemm}[\Cref{thm:kappa_bound}]
Suppose we observe $n = \Omega \big(\frac{k^2 \log \left(d/\nu\right)}{\eps}\big)$ $\eps$-corrupted samples from \Cref{equ:stat_model} with $\Sig = \Id$.
Let $\Input$ be an $\eps$-corrupted set of gradient samples $\{\gradsample_i^t\}_{i=1}^n$. \Cref{alg:filter-mean-second} computes $\lambda^*$ that satisfies
 \begin{align}\label{equ:lambda_lower_proof}
    \lambda^* \ge  \max_{\norm{\vect}_2=1, \norm{\vect}_0\leq \kk }\vect^{\top} \left({\Expe_{i \in_u \Input}
\left(\gradsample_i-\widehat{\UMean}\right)^{\otimes 2}}\right) \vect.
 \end{align}
If $\lambda^* \geq  \sep = C_\gamma \left(\norm{\UMean^t}_2^2 + \sigma^2\right)$, then with probability at least $1-\nu$, we have
\begin{align} \label{equ:UpperBoundGoodSample_proof}
\sum_{i \in \SGood} \tau_i \leq \tfrac{1}{{\gamma}} {\sum_{i \in \Input} \tau_i},
\end{align}
where $\tau_i$ is defined in line 10, $C_\gamma$ is a constant depending on $\gamma$, and $\gamma \geq 4$ is a constant.
\end{lemm}


%

\input{Filtering_proof.tex}

\section{RSGE via the filtering algorithm}
\label{sec:proof_Filter}

In this section, we still consider the $t$-th iteration of \Cref{alg:inexact_IHT} and  prove  \Cref{thm:Filter} on $t$. We omit $t$ in $\gradsample_i^t$ and $\UMean^t$.

In the case of $\lambda^* \geq C_\gamma  \left(\norm{\UMean}_2^2 + \sigma^2\right)$,
\Cref{alg:filter-mean-second} iteratively removes one sample according to the probability
distribution \cref{equ:Pr_remove}.
We denote the steps of this outlier removal procedure as $l = 1, 2, \cdots, n$.
The first step of proving \Cref{thm:Filter}
is to show we can remove a corrupted samples with high probability at each step,
which is a result by \Cref{thm:kappa_bound}.

Intuitively, if all subsequent steps are i.i.d., we can expect
\Cref{alg:filter-mean-second} to remove outliers within
 around $\eps n$ iterations, with  exponentially high probability.
However, the subsequent steps in \Cref{alg:filter-mean-second}
are not
independent.
To circumvent this challenge
we appeal to a martingale argument.

\subsection{Supermartingale construction}

Let $\Fil^l$
be the filtration generated by the set of events until iteration  $l$ of \Cref{alg:filter-mean-second}.
We define the corresponding set $\Input^l$, $\SGood^l$ and $\SBad^l$ at the step $l$.
We have that $\Input^l, \SGood^l, \SBad^l \in \Fil^l$, and $|\Input^l| = n-l$.

We denote a good event $\Event^l$ at step $l$ as
\begin{align*}
\sum_{i \in \SBad^l} \tau_i \leq \left(\gamma-1\right) \sum_{i \in \SGood^l} \tau_i.
\end{align*}
Then, by the definition of \Cref{alg:filter-mean-second} and \Cref{thm:kappa_bound}, if $\lambda^* \geq C_\gamma  \left( \norm{\UMean}_2^2 + \sigma^2\right)$, $\Event^l$ is not true;
if $\mathcal{E}^l$ is true, then \Cref{alg:filter-mean-second} will return a $\widehat{\UMean}$.

In \Cref{thm:outlier_prob}, we show that at any step $l$ when $\Event^l$ is not true, the random outlier removal procedure removes a
corrupted sample with probability at least $\left(\gamma-1\right)/\gamma$.

\begin{lemm} \label{thm:outlier_prob}
In each subsequent step $l$, if $\Event^l$ is not true, then we can remove one remaining outlier
from $\Input^l$ with probability at least $\left(\gamma-1\right)/\gamma$:
\begin{align*}
\Pr\left(\text{one sample from }   \SBad^l \text{ is removed }| \Fil_l\right) \geq \frac{\gamma-1}{\gamma}.
\end{align*}
\end{lemm}

\spro[Proof of \Cref{thm:outlier_prob}]
When $\lambda^* \geq C_\gamma  \left( \norm{\UMean}_2^2 + \sigma^2\right)$, \Cref{thm:kappa_bound} implies
\begin{align*}
\sum_{i \in \SBad^l} \tau_i \geq \left(\gamma-1\right) \sum_{i \in \SGood^l} \tau_i.
\end{align*}

Then we
randomly remove a sample $r$ from $\Input$ according to
\begin{align*}
\Pr\left(\gradsample_i \text{ is removed } | \Fil_l\right) = \frac{\tau_i}{\sum_{i \in \Input^l} \tau_i}.
\end{align*}
Finally,
\begin{align*}
\Pr\left(\text{one sample from }   \SBad^l \text{ is removed }| \Fil_l\right) =  \sum_{i \in \SBad^l}\frac{\tau_i}{\sum_{i \in \Input^l} \tau_i}   \geq \frac{\gamma-1}{\gamma}.
\end{align*}
\fpro

Since subsequent steps for applying
\Cref{alg:filter-mean-second} on $\Input$ are not
independent, we need martingale arguments
to show the total iterations of applying
\Cref{alg:filter-mean-second} is limited.

We use the martingale technique in
\cite{HRPCA2013},
by defining  $T$: $T = \min\{l: \mathcal{E}^l \text{ is true}\}$.
Based on $T$, we define a random variable:
\begin{align*}
\Mar^l = \begin{cases}
|\SBad^{T-1}| + \frac{\gamma - 1}{\gamma}\left(T-1\right), & \text{if } l\geq T\\
|\SBad^{l}| + \frac{\gamma - 1}{\gamma}l, & \text{if } l < T
\end{cases}
\end{align*}

\begin{lemm}[Lemma 1 in \cite{HRPCA2013}] \label{thm:martingale}
$\{\Mar^l, \Fil^l\}$ is a supermartingale.
\end{lemm}
Now, equipped with \Cref{thm:outlier_prob} and \Cref{thm:martingale}, we are ready to prove \Cref{thm:Filter}.

\subsection{Proof of \Cref{thm:Filter}}

\begin{theo}[\Cref{thm:Filter}]
Suppose we observe $n = \Omega \big(\frac{k^2 \log \left(d/\nu\right)}{\eps} \big)$ $\eps$-corrupted samples from \Cref{equ:stat_model} with $\Sig = \Id$.
Let $\Input$ be an $\eps$-corrupted set of gradient samples $\{\gradsample_i^t\}_{i=1}^n$.
By setting $\kk = k'+k$, if we run \Cref{alg:filter-mean-second} iteratively with initial set $\Input$, and subsequently on $\Output$, and use
$\sep = C_\gamma \big(\norm*{\UMean^t}_2^2 + \sigma^2\big)$,
then this repeated use of \Cref{alg:filter-mean-second} will stop after at most $\frac{1.1\gamma}{\gamma-1} \eps n$ iterations, and output $\widehat{\UMean}^t$, such that
\begin{align*}
\norm{\widehat{\UMean}^t - \UMean^t}_2^2 = \widetilde{O}\left(\eps \left( \norm{\UMean^t}_2^2+ \sigma^2\right)\right),
\end{align*}
with probability at least $1-\nu - \exp\left(-\Theta\left(\eps n\right)\right)$. Here, $C_\gamma$ is a constant depending on $\gamma$, where $\gamma \geq 4$ is a constant.
\end{theo}

\spro
We analyze \Cref{alg:filter-mean-second} by
discussing a series of $\{\Event^l\}$.

If $\Event^l$ is true, then $\lambda^* \leq \sep = C_\gamma\left( \norm{\UMean}_2^2 + \sigma^2\right)$.
By \Cref{lem:cov_sqeps},
we have
\begin{align*}
\lambda^* \ge  \max_{\norm{\vect}_2=1, \norm{\vect}_0\leq k }\vect^{\top} \left({\Expe_{i \in_u \Input}
\left(\gradsample_i-\widehat{\UMean}\right)^{\otimes 2}-F\left(\widehat{\UMean}\right)}\right) \vect \geq
\frac{\norm{{\Hard{\kk}\left(\widetilde{\meandiff}_S\right)}}_2^2}{5\epsilon}.
\end{align*}

Plugging in $\lambda^* \leq C_\gamma\left( \norm{\UMean}_2^2 + \sigma^2\right)$, we have
\begin{align*}
\frac{1}{5}{\norm{{\widehat{\meandiff}_S}}_2^2} \overset{(i)}\leq {\norm{{\Hard{\kk}\left(\widetilde{\meandiff}_S\right)}}_2^2} \leq {5\epsilon}\lambda^* \leq O\left(\eps \left( \norm{\UMean}_2^2 + \sigma^2\right)\right),
\end{align*}
where (i) follows from \Cref{thm:error_control}.
Hence, when $\Event^l$ is true, \Cref{alg:filter-mean-second} can return a $\widehat{\UMean}$, such that $\norm{\widehat{\UMean} - \UMean}_2^2 \leq O\left(\eps \left( \norm{\UMean}_2^2 + \sigma^2\right)\right)$.

%
%


Then, we only need to show $\bigcup_{l=1}^L \Event^l$ is true,
where $L = \frac{1.1 \gamma}{\gamma - 1} \eps n$, with
high probability.
That said, we need to upper bound the probability
\begin{align}\label{equ:tail}
\Pr\left(\bigcap_{l=1}^L \overline{\Event^l}\right) = \Pr\left(T \geq L\right) \leq  \Pr\left(\Mar^L \geq \frac{\gamma-1}{\gamma}L\right) = \Pr\left(\Mar^L \geq 1.1\eps n\right).
\end{align}

Then, we can construct the martingale difference according to \cite{HRPCA2013}.
Let $\Diff^l = \Mar^l - \Mar^{l-1}$, where $\Mar^0 = \eps n$, and
\begin{align*}
\bar{\Diff}^l = \Diff^l  - \Expe\left(\Diff^l | \Diff^1, \cdots, \Diff^{l-1}\right).
\end{align*}
Thus $\{\bar{\Diff}^l\}$ is a martingale difference process, and
$\Expe\left(\Diff^l | \Diff^1, \cdots, \Diff^{l-1}\right) \leq 0$, since $\{\Mar^l\}$ is a supermartingale.
Now, \cref{equ:tail} can be viewed as a bound for the sum of the associated martingale difference
sequence.
\begin{align*}
\Mar^l - \Mar^0 = \sum_{j=1}^{l} \Diff^j = \sum_{j=1}^{l} \bar{\Diff}^j +  \sum_{j=1}^{l} \Expe\left(\Diff^j | \Diff^1, \cdots, \Diff^{j-1}\right) \leq \sum_{j=1}^{l} \bar{\Diff}^j.
\end{align*}

Since we only remove one example from the set $\Input^l$, we can guarantee $|\Diff^l| \leq 1$ and $|\bar{\Diff}^l |\leq 2$.
For these bounded random variables, by applying
the Azuma-Hoeffding inequality, we have
\begin{align*}
\Pr\left(\Mar^L \geq 1.1\eps n\right) &\leq \Pr\left(\sum_{l=1}^{L}\bar{\Diff}^l \geq 0.1\eps n\right)\\
& \leq \exp\left(\frac{-\left(0.1\eps n\right)^2}{8L}\right).
\end{align*}
Plugging in $L = \frac{1.1 \gamma}{\gamma - 1}\eps n$, this probability is upper bounded by $\exp\left(-\Theta\left(\eps n\right)\right)$.

Notice that $L = \frac{1.1 \gamma}{\gamma - 1} \eps n \leq 1.5\eps n$, by setting
$\gamma \geq 4$.
Hence,
from $l=1$ to $L$,
we always have
$|\SBad^l| \leq 2\eps |\Input^l|$.
Then \Cref{thm:concentration_sqeps} and  \Cref{lem:cov_sqeps} hold and \Cref{thm:kappa_bound} is still valid.

Combining all of the above, we have proven that, with exponentially high probability,
\Cref{alg:filter-mean-second}  returns a $\widehat{\UMean}$ satisfying
$\norm{\widehat{\UMean} - \UMean}_2^2 \leq O\left(\eps \left( \norm{\UMean}_2^2 + \sigma^2\right)\right)$,
within $\frac{1.1\gamma}{\gamma-1} \eps n$ iterations.

\fpro

%% file: Filtering_proof.tex
\spro
Since $\lambda^*$ is the solution of the convex relaxation of Sparse PCA, we have
 \begin{align*}
    \lambda^*  =& \tr{H^* \cdot \left({\Expe_{i \in_u \Input}
\left(\gradsample_i-\widehat{\UMean}\right)^{\otimes 2}}\right)} \\
\ge &  \max_{\norm{\vect}_2=1, \norm{\vect}_0\leq \kk }\vect^{\top} \left({\Expe_{i \in_u \Input}
\left(\gradsample_i-\widehat{\UMean}\right)^{\otimes 2}}\right) \vect.
 \end{align*}

By Theorem A.1 in \cite{du2017computationally}, we have
\begin{align}\label{equ:22}
 \quad & \tr{ H^* \cdot \left( \Expe_{i \in_u \SGood}\left(\gradsample_i-\widehat{\UMean}\right)^{\otimes 2}
 - F\left(\widehat{\UMean}\right) \right)} \nonumber\\
& \leq C\left( \norm{\DSHat}_2^2 + \left(\LF+ \kk \norm{\widetilde{\meandiff}_{\SGood} }_{\infty} \right)\norm{\DSHat}_2
 + \kk\norm{ \Expe_{ i \in_u \SGood} \left(g_i - \UMean\right)^{\otimes 2} - F\left(\UMean\right)}_{\infty} \right),
\end{align}
where $C$ is a constant. Noticing that $\norm{\widetilde{\meandiff}_{\SGood} }_{\infty}$ and
$\norm{ \Expe_{ i \in_u \SGood} \left(g_i - \UMean\right)^{\otimes 2} - F\left(\UMean\right)}_{\infty}$  are unrelated to
$\widehat{\UMean}$ and only defined on $\SGood$, \cite{du2017computationally} shows concentration bounds for these two terms,
when $n = \Omega \left(\frac{\kk^2 \log (d/\nu)}{\eps}\right)$.
Specifically, it showed that  with probability
at least $1-\nu$, we have
\begin{align} \label{equ:classical_concentration}
\norm{\widetilde{\meandiff}_{\SGood} }_{\infty} & \leq C_1 \left(\LF + \sqrt{\Lcov}\right) \sqeps /\kk \\
\norm{ \Expe_{ i \in_u \SGood} \left(g_i - \UMean\right)^{\otimes 2} - F\left(\UMean\right)}_{\infty}& \leq  C_1 \left(\LF^2 + \Lcov\right) \sqeps /\kk
\end{align}

Now, we focus on the LHS of \cref{equ:UpperBoundGoodSample_proof}, the sum of scores of points in $\SGood$. By definition, we have
\begin{align*}
&\Expe_{i \in_u \SGood} \tau_i\\
&= \tr{ H^* \cdot \left( \Expe_{i \in_u \SGood}\left(\gradsample_i-\widehat{\UMean}\right)^{\otimes 2} \right)} \\
  &= \tr{ H^* \cdot \left( \Expe_{i \in_u \SGood}\left(\gradsample_i-\widehat{\UMean}\right)^{\otimes 2}
 - F\left(\widehat{\UMean}\right) \right)} + \tr{H^* F\left(\widehat{\UMean}\right)}\\
 & \overset{(i)}{\leq}  C\left( \norm{\DSHat}_2^2 + \left(\LF+ \kk \norm{\widetilde{\meandiff}_{\SGood} }_{\infty}\right)\norm{\DSHat}_2
 + \kk\norm{ \Expe_{ i \in_u \SGood}\left(g_i - \UMean\right)^{\otimes 2} - F\left(\UMean\right)}_{\infty} \right)\\
 &+ \tr{H^* F\left(\widehat{\UMean}\right)},
\end{align*}
where (i) follows from \cref{equ:22}.

To bound the RHS  above, we first bound $\tr{H^* \cdot F\left(\widehat{\UMean}\right)}$.
Because of the constraint of  the SDP
given in \cref{equ:convex_relax},  $H^*$ belongs to the Fantope $\mathcal{F}^1$~\cite{vu2013fantope}, and thus for any matrix $A$, we have
$\tr{A\cdot H^*} \leq \opnorm{A}.$

Thus, we have
\begin{align}
\tr{H^* \cdot F\left(\widehat{\UMean}\right)} &= \tr{H^* \cdot F\left({\UMean}\right)} + \tr{H^* *\left( F\left(\widehat{\UMean}\right)  - F\left({\UMean}\right)\right)} \nonumber\\
& \leq \opnorm{F\left({\UMean}\right)} + \opnorm{ F\left(\widehat{\UMean}\right)  - F\left({\UMean}\right)} \nonumber\\
& \overset{(i)}{\leq} C_1 \left(\norm{\UMean}_2^2 + \sigma^2\right) + \opnorm{ F\left(\widehat{\UMean}\right)  - F\left({\UMean}\right)}\nonumber \\
& \overset{(ii)}{\leq} C_1 \left(\norm{\UMean}_2^2 + \sigma^2\right) + \LF \norm{\DSHat}_2  + C_2 \norm{\DSHat}_2^2,\label{equ:bound_FG}
\end{align}
where (i) follows from the expression of $F\left(G\right)$ in \Cref{sec:calculations};
(ii) from  the smoothness  of $F\left(G\right)$.

By plugging in the concentration guarantees \cref{equ:classical_concentration}
and combining \cref{equ:bound_FG}, we have
\begin{align}
& \Expe_{i \in_u \SGood} \tau_i \nonumber \\
 &\leq   C_2 \left(
\left(\LF^2 + \Lcov\right) \sqeps + \left(\left(\LF + \sqrt{\Lcov}\right) \sqeps + \LF\right) \norm{\DSHat}_2 + \norm{\DSHat}_2^2
\right) + C_1 \left(\norm{\UMean}_2^2+\sigma^2\right) \nonumber\\
 &\overset{(i)}{\leq}  C_2  \left(
  \norm{\UMean}_2 \norm{\DSHat}_2 + \norm{\DSHat}_2^2
\right)+ C_1 \left(\norm{\UMean}_2^2 + \sigma^2\right) \nonumber\\
  &\leq C_1 \left(  \norm{\UMean}_2 \norm{\DSHat}_2 + \norm{\DSHat}_2^2 + \norm{\UMean}_2^2 + \sigma^2
\right), \label{equ:bound_tau}
\end{align}
where (i) follows from the fact that $\eps$ is sufficiently small.

On the other hand, we know that:
$\Expe_{i \in_u \Input} \tau_i = \lambda^*.$



Now, under the condition 
$\lambda^* \geq  \sep = \Theta\left(\norm{\UMean}_2^2 + \sigma^2\right)$, we consider two cases separately.
{By separating two cases, we can always show $\lambda^*$ is very large, and the contribution from good samples is limited.}

First,
if $\norm*{\DSHat}_2^2  \geq \Theta \left(\norm{\UMean}_2^2 + \sigma^2\right)$,
then in \cref{equ:bound_tau}, we have
\begin{align*}
   \norm*{\DSHat}_2^2 \gtrsim
   \norm{\UMean}_2  \norm*{\DSHat}_2 \gtrsim
   \norm{\UMean}_2^2,
   \quad \text{ and } \norm*{\DSHat}_2^2\gtrsim \sigma^2.
\end{align*}
Thus, we only need to compare $\lambda^*$ and $\norm*{\DSHat}_2^2$.
By \Cref{lem:cov_sqeps}, we have
\begin{align*}
\Expe_{i \in_u \Input} \tau_i = \lambda^* &\ge \max_{\norm{\vect}_2=1, \norm{\vect}_0\leq \kk }\vect^{\top} \left({\Expe_{i \in_u \Input}
\left(\gradsample_i-\widehat{\UMean}\right)^{\otimes 2}}\right) \vect \\
&\ge \max_{\norm{\vect}_2=1, \norm{\vect}_0\leq \kk }\vect^{\top} \left({\Expe_{i \in_u \Input}
\left(\gradsample_i-\widehat{\UMean}\right)^{\otimes 2}-F\left(\widehat{\UMean}\right)}\right) \vect \\
&\geq
\frac{\norm{\DSHat}_2^2}{\epsilon}.
\end{align*}
Hence, by \cref{equ:bound_tau}, we have $\Expe_{i \in_u \Input} \tau_i\geq \gamma \Expe_{i \in_u \SGood} \tau_i$, where $\gamma \geq 4$ is a constant.

Second, if $\norm*{\DSHat}_2^2 \leq  \Theta \left(\norm{\UMean}_2^2 + \sigma^2\right)$, then in \cref{equ:bound_tau}, we have
\begin{align*}
    \norm{\UMean}_2^2\gtrsim  \norm{\UMean}_2  \norm*{\DSHat}_2 \gtrsim \norm*{\DSHat}_2^2,
   \quad \text{ or } \sigma^2 \gtrsim \norm*{\DSHat}_2^2.
\end{align*}
Thus, we only need to compare $\lambda^*$ and $\max\left(\norm{\UMean}_2^2, \sigma^2\right)$.
Since $\lambda^* \geq  C_\gamma \left(\norm{\UMean}_2^2 + \sigma^2\right)$ by the condition of \Cref{thm:kappa_bound},
we still have $\Expe_{i \in_u \Input} \tau_i\geq \gamma \Expe_{i \in_u \SGood} \tau_i$, where $\gamma \geq 4$ is a constant.

Combing all of above, and  setting
$\sep = C_\gamma\left(\norm{\UMean}_2^2 + \sigma^2\right)$, we have
\begin{align*}
\sum_{i \in \Input} \tau_i  = |\Input|\Expe_{i \in_u \Input}\tau_i \geq {\gamma} |\SGood|\Expe_{i \in_u \SGood}\tau_i = {\gamma} {\sum_{i \in \SGood} \tau_i}.
\end{align*}
\fpro

%% file: Unknown_proof.tex
\section{Robust sparse regression with unknown covariance}
\label{sec:proof_unknown}

In this section, we prove the guarantees for RSGE when  the covariance matrix $\Sig$ is unknown, but each row and column is sparse. 
 In this case, the population mean of all authentic gradients $\UMean^t$ can be calculated as
\begin{align*}
\UMean^t = \Expe_P\left(\gradsample_i^t\right) = \Expe_P\left(\x_i \x_i^{\top}\left(\param^{t} - \param^{*}\right)\right) = \Sig \paramdiff^t.
\end{align*}
Therefore, $\UMean^t = \Sig \paramdiff^t$ is guaranteed to be $r(k'+k)$ sparse.
And we use  the filtering algorithm (\Cref{alg:filter-mean-second}) with $\kk = r(k'+k)$ as a RSGE.

First, we derive the functional $F\left(\UMean\right)$ with general covariance matrix $\Sig$, and
compute the corresponding $\LF, \Lcov$, which has been defined  in \cref{equ:LCOV} and \cref{equ:LF}  for the case $\Sig = \Id$ in \Cref{sec:calculations}.


\begin{lemm}\label{lem:F_Unknown} 	
Suppose we observe i.i.d. samples $\{\bm{z}_i, i\in\Good\}$ from the distribution P in \Cref{equ:stat_model} with an unknown $\Sig$, we have the covariance of gradient as
\begin{align*}
\Cov(\gradsample) := 
\Expe_{\bm{z}_i \sim P} \left( \left(\gradsample_i - \UMean \right)\left(\gradsample_i - \UMean \right)^{\top} \right) = \Sig \norm{\Sigsqn \UMean}_2^2  +  \UMean \UMean^\top +  \sigma^2 \Sig.
\end{align*}
\end{lemm}

\spro
As in the \Cref{equ:stat_model}, we draw $\bm{x}$ from Gaussian distribution $\mathcal{N}(0, \Sig)$,
the expression of  $F(\cdot)$ is  given by
\begin{align*}
\Cov(\gradsample) &=  \Expe\left( \left(\gradsample_i - \UMean \right)\left(\gradsample_i - \UMean \right)^{\top}\right)\\
 &= \Expe\left( \left(\x\x^{\top} - \Sig\right) \paramdiff\paramdiff^{\top} \left(\x\x^{\top} - \Sig\right)\right) + \sigma^2 \Sig\\
 &\overset{(i)}{=} \Expe\left( \Sigsq \left(\xn\xn^{\top} - \Id\right) \Sigsq\paramdiff\paramdiff^{\top}\Sigsq \left(\xn\xn^{\top} - \Id\right)\Sigsq \right) + \sigma^2 \Sig\\
 & \overset{(ii)}{=} \Sig \norm{\Sigsqn \UMean}_2^2  +  \UMean \UMean^\top +  \sigma^2 \Sig.
\end{align*}
where (i) follows from the re-parameterization $\bm{x} = \Sigsq \widetilde{\bm{x}}$, where $\widetilde{\bm{x}} \sim \mathcal{N}(0, \Id)$, and (ii) follows from the Stein-type Lemma as in \Cref{sec:calculations}.
\fpro

By \Cref{lem:F_Unknown}, we define the functional $F(\UMean) = \Sig \norm{\Sigsqn \UMean}_2^2  +  \UMean \UMean^\top +  \sigma^2 \Sig$.
In \Cref{alg:filter-mean-second}, we do not need to evaluate $F(\cdot)$, but our analysis requires upper bounds for   two parameters of 
  $F(\cdot)$ -- 
    $\Lcov, \LF$ -- to control   
  tail bounds. 
 Under the same setting as \Cref{lem:F_Unknown}, we use similar bounds as  \Cref{sec:calculations}, based on   assumptions in \Cref{equ:stat_model}. Hence, we have
 $\Lcov = \Theta( \norm{\UMean}_2^2 + \sigma^2)$, and $\LF = \Theta( \norm{\UMean}_2)$.

Next, we show concentration bounds (\Cref{thm:concentration_sqeps_heavy}) similar to
\Cref{thm:concentration_sqeps}, which controls deviation of empirical mean and covariance for 
all samples in the good set $\Good$.

\begin{lemm}
\label{thm:concentration_sqeps_heavy}
Suppose we observe i.i.d. gradient samples $\{\gradsample_i, i\in\Good\}$ from
\Cref{equ:stat_model} with
$\abs{\Good} =\widetilde{\Omega}\left(\frac{\kk\log\left(d/\nu\right)}{\epsilon}\right)$.
Then,
there is a $\delta = \widetilde{O}\left(\sqeps\right)$, such that with probability at least $1-\nu$, for any index subset ${\Ind} \subset [d]$, $|{{\Ind}}| \le \kk$
 and for any
$\Good' \subset \Good$, $\abs{\Good'} \geq (1-2\eps)\abs{\Good}$,
we have
\begin{align}
	\norm{\Expe_{i \in_u \Good'} \left(\gradsample_i^{ {\Ind}} \right) -\UMean^{ {\Ind}} }_2 &
\leq \delta\left(\norm{{\UMean}}_2 + \sigma\right), \label{equ:gradient_mean_heavy}\\
\norm{{\Expe_{i \in_u \Good'} \left(\gradsample_i^{ {\Ind}}-\UMean^{ {\Ind}}\right)} ^{\otimes 2} -  F\left(\UMean\right)^{\Ind \Ind}}_{\rm op} & \leq \delta\left(\norm{{\UMean}}_2^2 + \sigma^2\right).  \label{equ:gradient_covariance_heavy}
\end{align}
\end{lemm}

\spro
We prove the concentration inequality for the covariance  \cref{equ:gradient_covariance_heavy}, the bound for mean  \cref{equ:gradient_mean_heavy} is similar.
For any index subset ${\Ind} \subset [d]$, $|{{\Ind}}| \le \kk$,
we can expand  \cref{equ:gradient_covariance_heavy} as follows,
\begin{align}
    &{\Expe_{i \in_u \Good'} \left(\gradsample_i^{ {\Ind}}-\UMean^{ {\Ind}}\right)} ^{\otimes 2} -  F\left(\UMean\right)^{\Ind \Ind} \nonumber\\
&= \Expe_{i \in_u \Good'} \left( \x^\Ind \x^{\top} \paramdiff \paramdiff^{\top} \x (\x^\Ind)^{\top} \right) -
\left(\Sig^{{\Ind \Ind}} \norm{\Sigsq \paramdiff }_2^2  +  2 \UMean^\Ind (\UMean^\Ind)^\top \right) \label{equ:term_1}\\
&-   \Expe_{i \in_u \Good'} \left( \x \x^{\top} \paramdiff \paramdiff^{\top} \Sig \right)^{\Ind \Ind} + \UMean^\Ind (\UMean^\Ind)^\top \label{equ:term_2}\\
&+  \Expe_{i \in_u \Good'} \xi_i^2  \x^\Ind (\x^\Ind)^{\top} - \sigma^2 \Sig^{\Ind \Ind} \label{equ:term_3}
\end{align}

Here, we prove the concentration inequality for \cref{equ:term_1}, and the other two terms
can be bounded by the same technique.
It is sufficient to prove an upper bound for the operator norm as follows
\begin{align}
\opnorm{ \Expe_{i \in_u \Good'} \x^\Ind (\x^\Ind)^{\top} \paramdiff^\Ind (\paramdiff^\Ind)^{\top} \x^\Ind (\x^\Ind)^{\top} - \left(\Sig^{{\Ind \Ind}} \norm{\Sigsq \paramdiff }_2^2  +  2 \UMean^\Ind (\UMean^\Ind)^\top \right)} \leq \delta \norm{{\UMean}}_2^2,\label{equ:covariance_heavy}
\end{align}
where $\x$ is drawn from a Gaussian distribution $\mathcal{N} (0, \Sig)$.
Note that the index  subset $\Ind$ reduce the  matrix
to $\Real^{\abs{\Ind} \times \abs{\Ind}}$.
For the concentration bounds of covariance matrix estimation \cref{equ:covariance_heavy}, we have a near identical argument as Lemma 4.5 of \cite{Moitra2016FOCS}, 
by  replacing Theorem 5.50 with Theorem 5.44 in \cite{vershynin2010introduction}.


This establishes  \cref{equ:covariance_heavy} with sample complexity
$n=\widetilde{\Omega}\left(\frac{\kk \log\left(1/\nu\right)}{\epsilon}\right)$,
with probability at least $1-\nu$.
Next, we take a union bound
over all possible subsets $\Ind \subset [d]$, and
this gives  concentration results for the covariance
\cref{equ:gradient_covariance_heavy}. Hence we have proved the concentration results for the gradient under the assumption that $\Sig$ is row/column sparse.
\fpro

Based on  \Cref{thm:concentration_sqeps_heavy},
 we have \Cref{coro:unknown}, which guarantees the recovery of $\param^*$ in robust sparse regression with unknown covariance as defined in \Cref{model:covariance}.

\begin{coro}[\Cref{coro:unknown}]
Suppose we observe $N\left(k, d, \eps, \nu\right)$ $\eps$-corrupted samples from
\Cref{equ:stat_model}, where the covariates $\bm{x}_i$'s follow from \Cref{model:covariance}.
If we use \Cref{alg:filter-mean-second} for robust sparse gradient estimation, it requires 
$\widetilde{\Omega}\left(\frac{r^2 k^2 \log \left(dT/\nu\right)}{\eps}\right) T $ samples,
and  $T = \Theta\left(\log\left(\frac{\norm{\param^*}_2}{\sigma\sqeps}\right)\right)$,
then, we have
\begin{align}
\norm{\widehat{\param} - \param^*}_2 = \widetilde{O}\left(\sigma\sqeps\right),
\end{align}
with probability at least $1- \nu - T \exp\left(-\Theta\left(\eps n\right)\right)$.
\end{coro}

\spro
With the concentration result \Cref{thm:concentration_sqeps_heavy} in hand,
the remaining parts share the same theoretical analysis as   \Cref{sec:proof_kappa_bound}
and \Cref{sec:proof_Filter}, by replacing $(k'+k)^2$ with $r^2(k'+k)^2 = \Theta (r^2k^2)$.
Hence, we have a result similar to \Cref{thm:coro_2}, with sample complexity $\widetilde{\Omega}\left(\frac{r^2 k^2 \log \left(dT/\nu\right)}{\eps}\right)$.
And this yields \Cref{coro:unknown}.
\fpro

%% file: Experiment.tex
\newpage

\section{Numerical results}\label{sec:simulation}


\subsection{Robust sparse mean estimation}
\label{sec:exp_mean}

We first demonstrate the performance of \Cref{alg:filter-mean-second} for robust sparse mean estimation, and then move to \Cref{alg:inexact_IHT} for robust sparse regression.
For the robust sparse gradient estimation, we generate samples through  $\gradsample_i = \x_i \x_i^{\top}\UMean - \x_i \xi_i$, where the unknown true mean $\UMean$ is $k$-sparse.
The authentic $\x_i$'s are generated from $\mathcal{N} (0, \Id)$. We set $\sigma = 0$, since the main part of the error in robust sparse mean estimation is $\UMean$.  Each entry of $\UMean$ is either $+1$ or $-1$, hence $\norm{\UMean}_2^2 = k$.

The outliers are specially designed: the norm of the outliers is  $\norm{\UMean}_2$, and the directions are orthogonal to
$\UMean$. Through this construction, outliers cannot be easily removed by simple pruning, and
the directions of outliers can cause large effects on
the estimation of $\UMean$. We plot the relative MSE of parameter recovery,  defined as
$\|\widehat{\UMean} - \UMean\|_2^2 / \norm{\UMean}_2^2$, with respect to different
sparsities and dimensions.

\textbf{Parameter error vs. sparsity $k$.}
We fix the dimension to be $d = 50$. We solve the trace norm maximization in \Cref{alg:filter-mean-second} using CVX \cite{grant2008cvx}.
We solve robust sparse gradient estimation
under different levels of outlier fraction $\eps$
and  different sparsity values $k$.

\textbf{Parameter error vs. dimension $d$.}
We fix $k = 5$. We use a Sparse PCA solver from \cite{PathSPCA} which is much more efficient for higher dimensions. We run robust sparse gradient estimation \Cref{alg:filter-mean-second}
under different levels of outlier fraction $\eps$
and  different dimensions $d$.

For each parameter, the corresponding number of samples required for the authentic data is $n \propto k^2 \log(d) /\eps$ according to \Cref{thm:Filter}. Therefore, we add $\eps n/(1 - \eps)$ outliers (so that the outliers are an $\eps$-fraction of the total samples),
and then run \Cref{alg:filter-mean-second}. 
According to
\Cref{thm:Filter}, the rescaled relative MSE:
$\|\widehat{\UMean} - \UMean\|_2^2 / (\eps \norm{\UMean}_2^2)$ should be independent of
the parameters $\{\eps, k, d\}$. We plot this
in \Cref{fig:RSGE}, and
these plots validate our theorem on the sample complexity in robust sparse mean estimation problems.

\begin{figure}[t]
\centering
{
\includegraphics[width=.38\columnwidth]{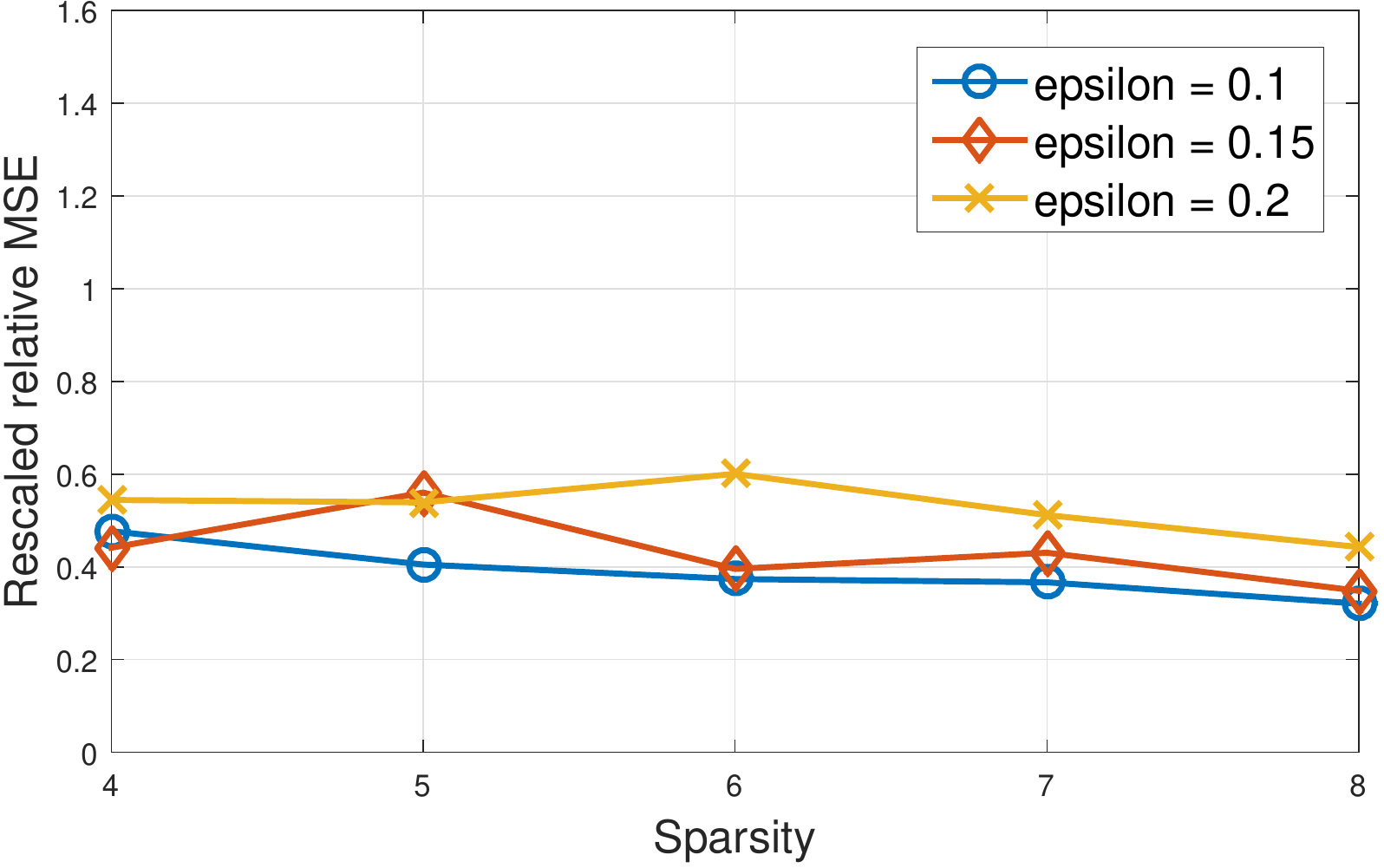}
}{
\includegraphics[width=.4\columnwidth]{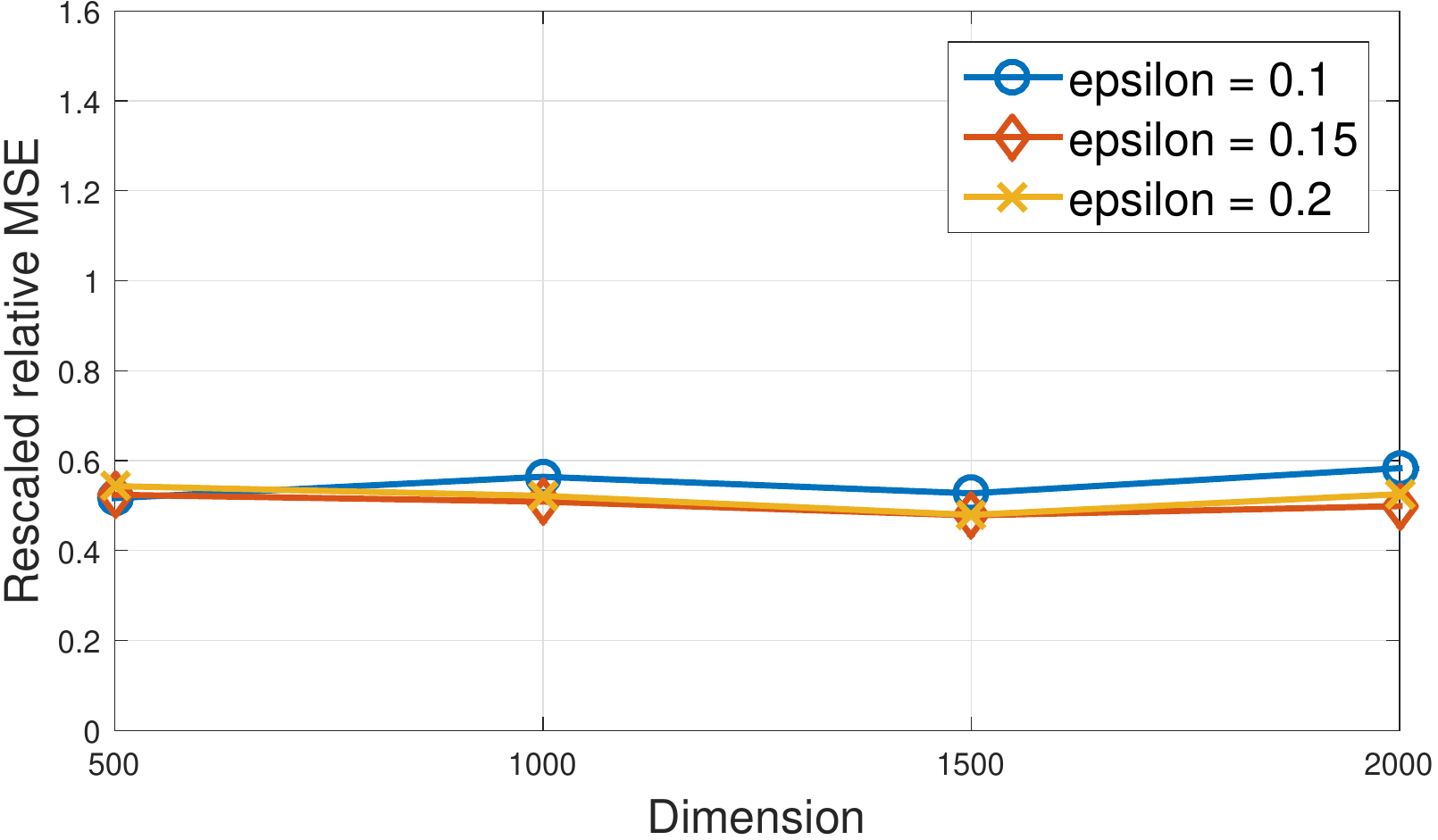}
}
\caption{
\footnotesize{Simulations for \Cref{alg:filter-mean-second} showing the dependence of
relative MSE on sparsity and dimension.
For each parameter, we choose corresponding sample complexity
$n \propto k^2 \log(d)/\eps$.
Different curves for $\eps \in \{0.1, 0.15, 0.2\}$ are the average of 15 trials.
Consistent with the theory, the rescaled relative MSE's are nearly independent of sparsity and dimension. Furthermore, by rescaling for different $\eps$, three curves have the same magnitude.
}}
\label{fig:RSGE}
\end{figure}

\subsection{Robust sparse regression with identity covariance}
\label{sec:exp_regression}
We use \Cref{alg:inexact_IHT} for robust sparse regression. Similarly as in \Cref{sec:exp_mean}, we use
\Cref{alg:filter-mean-second} as our Robust Sparse Gradient Estimator, and leverage the Sparse PCA solver from \cite{PathSPCA}.
In the simulation, we fix $d = 500$, and $k = 5$, hence the corresponding sample complexity is $n\propto 1/\eps$. However, we do not use the sample splitting technique in the simulations.

The entries of the true parameter $\param^*$  are set to be either $+1$ or $-1$, hence $\norm{\param^*}_2^2 = k$ is fixed.
The authentic $\x_i$s are generated from $\mathcal{N} (0, \Id)$, and the authentic $y_i = \x_i^{\top} \param^* + \xi_i$ as in \Cref{equ:stat_model}.
We set the covariates of the outliers as $A$, where $A$ is a random $\pm 1$ matrix of dimension $\eps n /(1-\eps) \times d $, and set the responses of outliers to $- A\param^*$.

To show the performance of \Cref{alg:inexact_IHT} under different settings, we use different levels of $\eps$ and $\sigma$ in \Cref{fig:Robust sparse regression}, and track the parameter error $\norm{\param^t - \param^*}_2^2$ of \Cref{alg:inexact_IHT} in each iteration.
Consistent with the theory, the algorithm displays linear convergence, and the error curves
flatten out at the level of the final error. Furthermore, \Cref{alg:inexact_IHT} achieves machine precision when $\sigma^2 = 0$ in the right plot of \Cref{fig:Robust sparse regression}.



\begin{figure}[t]
\centering
\includegraphics[width=.38\columnwidth]{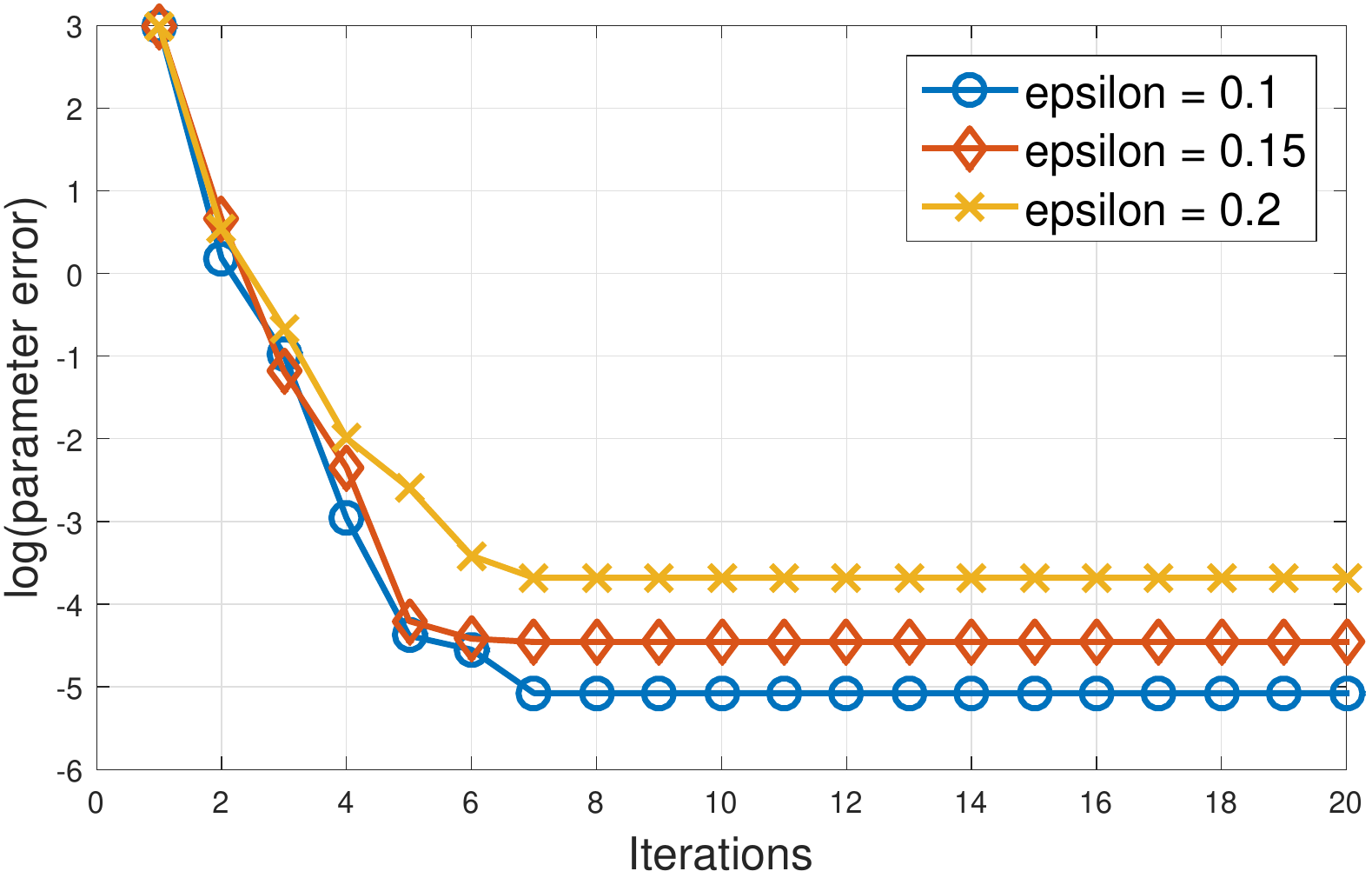}
\includegraphics[width=.4\columnwidth]{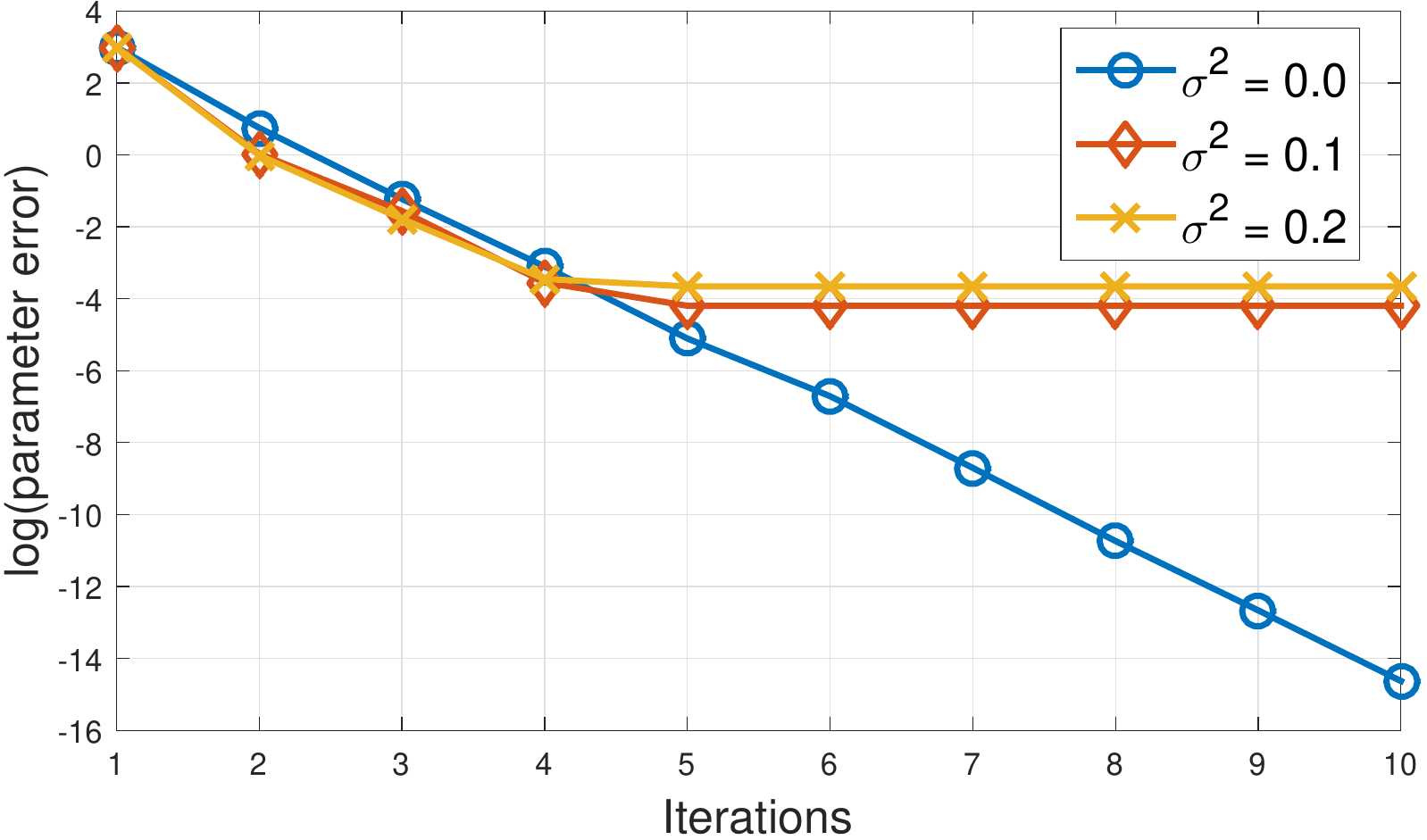}
\caption{\footnotesize{Empirical illustration of the linear convergence of
$\log(\|\param^t - \param^*\|_2^2)$ vs. iteration counts in the \Cref{alg:inexact_IHT}.
In all cases, we fix  $k = 5$, $d = 500$, and choose the sample complexity $n\propto 1/\eps$.
The left plot considers different
$\eps$ with fixed $\sigma^2=0.1$.
The right plot considers different
$\sigma^2$ with fixed $\eps=0.1$.
As expected, the convergence is linear, and
flatten out at the level of the final error.
}}
\label{fig:Robust sparse regression}
\end{figure}






\subsection{Robust sparse regression with unknown  covariance matrix}

Following 
\Cref{sec:exp_regression},
we study the empirical performance of robust sparse regression with unknown covariance matrix $\Sig$ following from \Cref{model:covariance}.

We use the same experimental setup as in \Cref{sec:exp_regression}, but modify the covariance matrix to be a Toeplitz matrix with a decay
$\Sig_{ij} = \exp^{-(i-j)^2}$.
Under this setting, the covariance matrix is sparse, thus follows from \Cref{model:covariance}.
\Cref{fig:Heavy} indicates that we have nearly the
same performance as the $\Sig = \Id$ case.

\begin{figure}[t]
\centering
\includegraphics[width=.395\columnwidth]{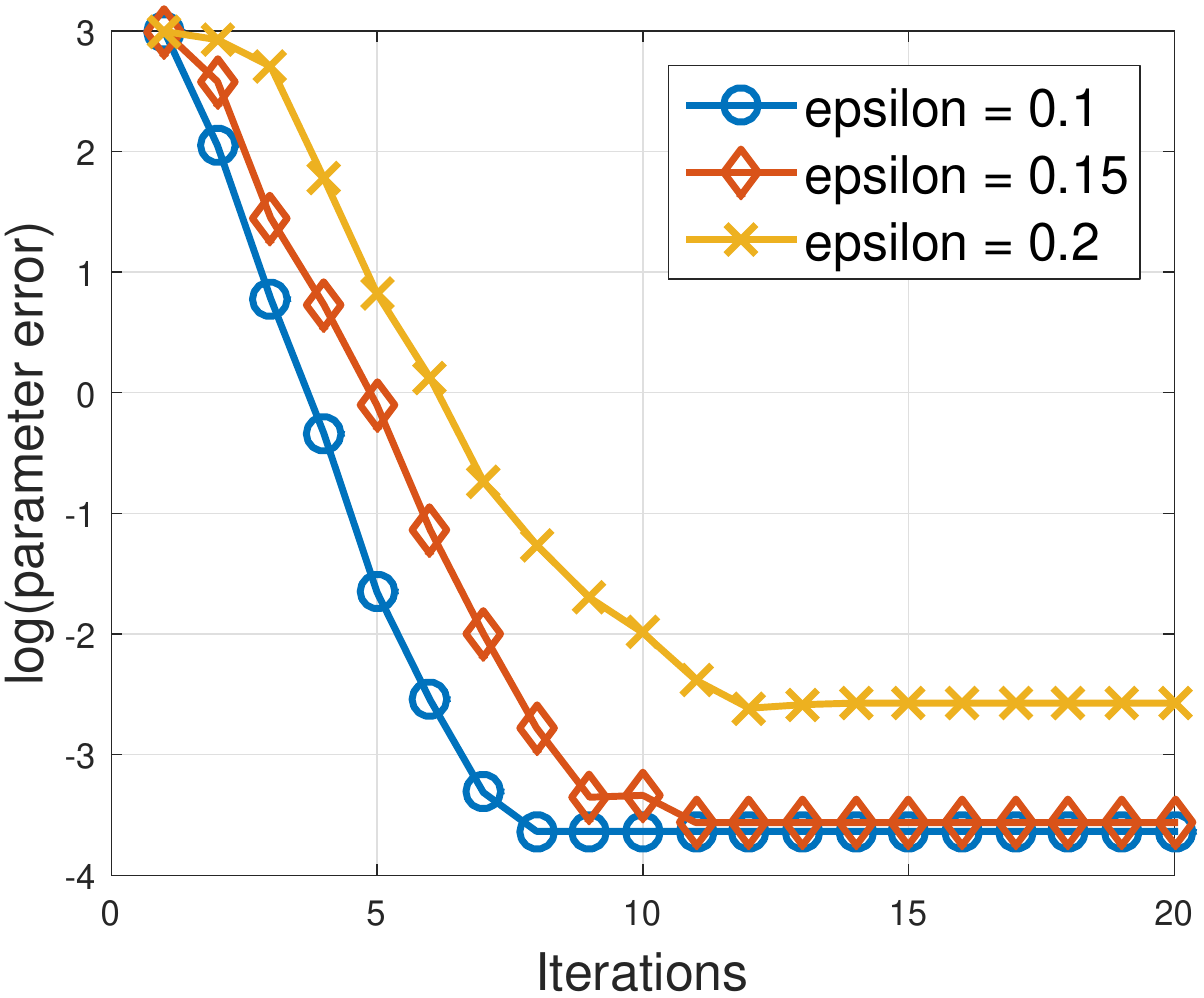}
\includegraphics[width=.4\columnwidth]{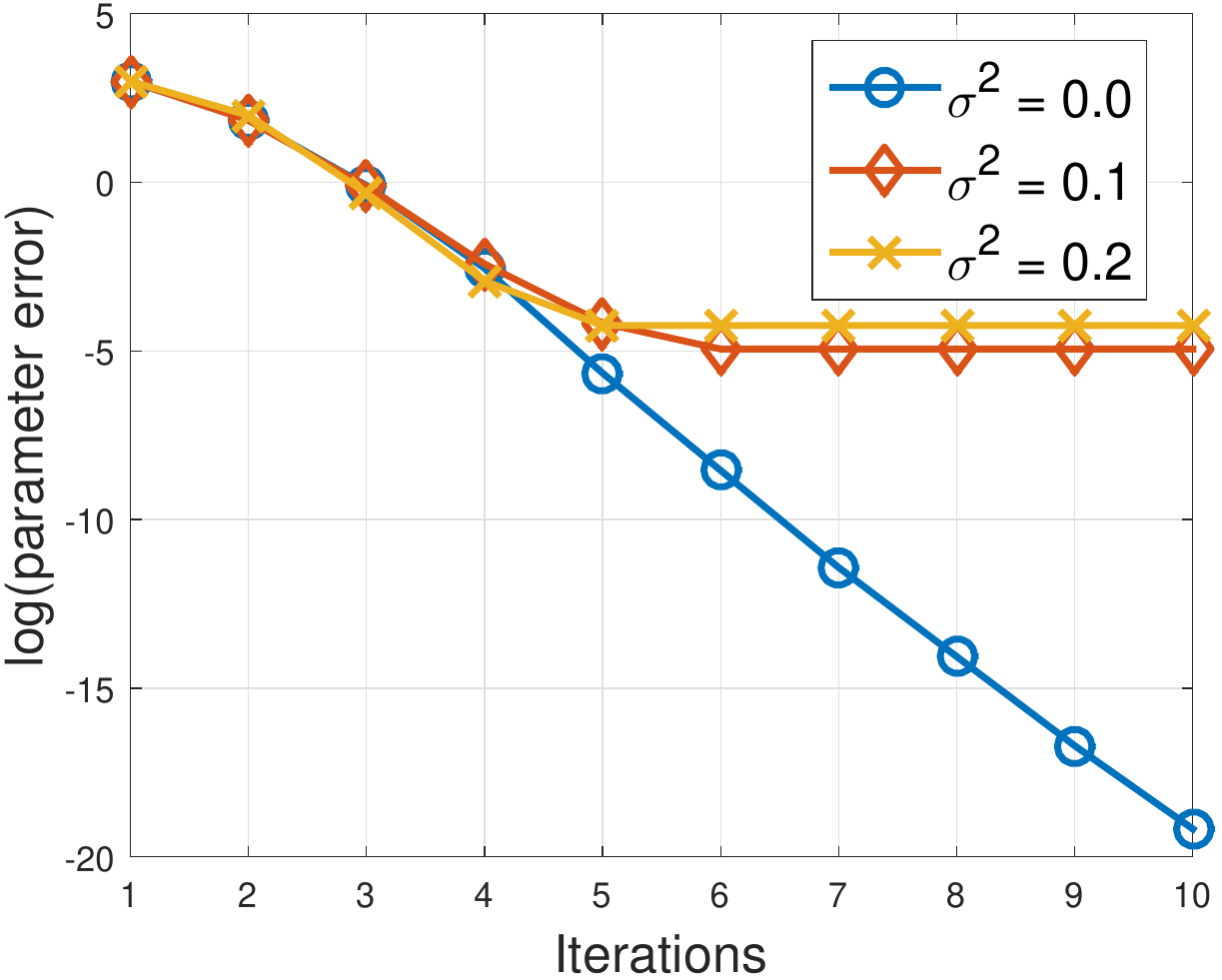}
\caption{\footnotesize{Empirical illustration of the linear convergence of
$\log(\|\param^t - \param^*\|_2^2)$ vs. iteration counts in the \Cref{alg:inexact_IHT} with unknown covariance matrix which is a
Toeplitz matrix with a decay
$\Sig_{ij} = \exp^{-(i-j)^2}$.
The other settings are the same as \Cref{fig:Robust sparse regression}.
Even though the covariance matrix is unknown, we observe similar performance in linear convergence as \Cref{fig:Robust sparse regression}.
}}
\label{fig:Heavy}
\end{figure}